\documentclass[11pt,twoside]{article}
\usepackage{ltexpprt}
\usepackage{preamble}
\usepackage{authblk}
\usepackage[capitalize,nameinlink,noabbrev]{cleveref}

\makeatletter
\newcommand{\leqnomode}{\tagsleft@true\let\veqno\@@leqno}
\newcommand{\reqnomode}{\tagsleft@false\let\veqno\@@eqno}
\makeatother

\begin{title}
{\LARGE Near-Optimal Degree Testing for Bayes Nets}
\end{title}

\author[1]{Vipul Arora\thanks{Vipul was supported in part by NRF-AI Fellowship R-252-100-B13-281.}}
\affil{\footnotesize National University of Singapore. $\{$\texttt{\href{mailto:vipul@comp.nus.edu.sg}{vipul}, \href{mailto:arnab@comp.nus.edu.sg}{arnab}, \href{mailto:jyang@comp.nus.edu.sg}{jyang}}$\}$\texttt{@comp.nus.edu.sg}.}
\author[1]{Arnab Bhattacharyya\thanks{A. Bhattacharyya was supported in part by an MOE Tier II award, and an Amazon faculty research award.}}
\author[2]{Cl\'ement L. Canonne}
\affil{\footnotesize University of Sydney. \texttt{\href{mailto: clement.canonne@sydney.edu.au}{ clement.canonne@sydney.edu.au},\;\href{mailto:qyan6238@uni.sydney.edu.au}{qyan6238@uni.sydney.edu.au}}}
\author[1,2]{Joy Qiping Yang\thanks{Joy was supported in part by NRF-AI Fellowship R-252-100-B13-281.}}

\date{}

\begin{document}
\maketitle

\begin{abstract} 
This paper considers the problem of testing the maximum in-degree of the Bayes net underlying an unknown probability distribution $P$ over $\{0,1\}^n$, given sample access to $P$. We show that the sample complexity of the problem is $\tilde{\Theta}(2^{n/2}/\varepsilon^2)$. Our algorithm relies on a testing-by-learning framework, previously used to obtain sample-optimal testers; in order to apply this framework, we develop new algorithms for ``near-proper'' learning of Bayes nets, and high-probability learning under $\chi^2$ divergence, which are of independent interest.
\end{abstract}

\section{Introduction}
\label{sec:introduction}
One of the most natural and widely-used ways to model high-dimensional distributions is as {\em Bayesian networks} (or, {\em Bayes nets} for short) \cite{pearl1988probabilistic}. In particular, a Bayes net on $\{0,1\}^n$ is given by a directed acyclic graph (DAG) $G$ on $n$ vertices, and probability distributions $p_{i,\pi}$ on $\{0,1\}$, for all $i \in [n]$, and all assignments $\pi$ to the parents of the $i$'th node in the graph $G$. To generate an $n$-dimensional sample, one samples the nodes in a topological order of $G$, where the $i$'th node is sampled according to $p_{i, \pi}$ for the assignment $\pi$ that is already fixed by the samples for the parent nodes of $i$. Bayes nets naturally encode causal information \cite{pearl1995bayesian}, and learning a Bayes net description of a probability distribution is considered a fundamental problem in statistics and machine learning \cite{heckerman1998tutorial}. For instance, Sachs et al.~\cite{sachs2005causal} used this approach to discover protein regulatory networks from gene expression data.  

Sufficiently complex Bayes nets can encode arbitrary distributions, and so, it is infeasible in general to learn general Bayes nets. Instead, one often restricts to Bayes nets whose underlying DAGs have bounded {\em in-degree}. Distributions having sparse Bayes net descriptions naturally arise in machine learning, robotics, natural language processing~\cite{wainwright2008graphical}, medicine, and computational biology \cite{friedman2000using}. Moreover, information-theoretically, it is known that Bayes nets with in-degree bounded by $d$ can be learned up to total variation (TV) distance $\eps$ using $\tilde{O}(n2^d/\eps^2)$ samples \cite{canonne2017testing}. Hence, the maximum in-degree of a Bayes net is an important modeling parameter to consider.

In this work, we consider the problem of testing whether a distribution belongs to the concept class of degree-$d$ Bayes nets, i.e.,~those whose in-degree is at most $d$. Specifically, given sample access to a distribution $P$ on $\{0,1\}^n$, we consider the property testing question: is $P$ described by a degree-$d$ Bayes net, or is $P$ $\eps$-far from all such Bayes nets, in TV distance? 

\begin{thm}[Main Theorem, Informal, See \autoref{theorem:graph-test-general-distribution}]
  \label[Theorem]{theorem:informal:degree_testing}
  Given an unknown distribution $P$ on $\{0,1\}^n$ and a degree parameter $d \ll n$, testing whether $P$ is Markov with respect to any degree-$d$ Bayes net has sample complexity $\tilde{\Theta} \left( \frac{2^{n / 2}}{\varepsilon^2} \right )$.
\end{thm}
Our result requires that $d < n/2 - \Omega(\log n)$, which is consistent with our motivation of testing Bayes net sparsity. 
The main contribution of this work lies in establishing the upper bound of the theorem; we note that the (nearly) matching lower bound follows from \cite[Corollary~B.2]{DBLP:journals/corr/abs-2204-08690}.
\ignore{
\begin{itemize}
    \item In-degree of Bayes nets is a way to quantify the complexity of the underlying distribution. In particular, it often appears in many sample complexity bounds, e.g., it takes $\tilde{\Theta}(2^d n / \varepsilon^2)$ for learning a degree-$d$ Bayes net, $\tilde{\Theta}(2^{d/2} n / \varepsilon^2)$ for testing independence of a degree-$d$ Bayes net.
    \item Extensions of independence testing.
    \item Maybe important to motivate low degree testing: $d \leqslant n/2 - \log(n)$.
\end{itemize}

We emphasize that when stating our results regarding to $\chi^2$ learning, the distance parameter is stated as $\eps$; however, due to the quadratic relation between total variation (TV) distance and $\chi^2$ divergence, the corresponding results relating to testing in TV distance will feature an $\eps^2$ dependence.}

Over the course of deriving this upper bound, we obtain three new learning results which we believe to be of independent interest. (1)~First, a \emph{high-probability learning} result in $\chi^2$ divergence (\cref{prop:chi_square_learn}). While the sample complexity of learning an arbitrary distribution $P$ over a discrete domain $\Sigma$ to $\chi^2$ divergence $\eps$ is known to be $\Theta(|\Sigma|/\eps)$ for \emph{constant} probability of success,\footnote{i.e., outputting $\hat{P}$ such that $d_{\chi^2} (P, \hat{P}) \leq \eps$ with probability at least $9/10$.} boosting this success probability to $1-\delta$ for arbitrarily small $\delta>0$ was until now open. In particular, whether a \emph{logarithmic} dependence on $\delta$ was possible, as for total variation distance and (as recently shown) KL divergence, was unknown. We show this is indeed the case: $O\left( \frac{| \Sigma |}{\varepsilon} \log   \frac{| \Sigma |}{\delta} \right)$ i.i.d.\ samples suffice for high-probability learning in $\chi^2$ divergence. Interestingly, the estimator achieving this bound is neither the empirical estimator, nor the usual Laplace add-$1$ estimator, but instead an add-$K$ estimator for a suitable $K=K(\delta)$.

(2)~Second, a \emph{near-proper} learning algorithm in $\chi^2$-divergence for degree-$d$ Bayes nets, with sample complexity $\tilde{O}(2^d n^2/\eps)$ (\autoref{theorem:near_proper_chi_squared_learning_bayes}). We note that previous learning algorithms for Bayes nets either learn with respect to a \emph{weaker distance measure} (TV or KL) or are \emph{non-proper}, in the sense that the hypothesis they output is not a degree-$d$ Bayes net itself. In comparison, our algorithm outputs a \textit{bona fide} degree-$d$ Bayes net $\hat{P}$, along with a subset $S\subseteq\{0,1\}^n$ of the domain such that (i)~$P$, $\hat{P}$ put all but $O(\eps)$ probability mass on $S$, and (ii)~$P$, and $\hat{P}$ are within $\chi^2$ divergence $\eps$ when \emph{restricted to this subset $S$}. 

This hybrid guarantee, which makes our learning algorithm \emph{near}-proper instead of proper, may seem artificial. However, our third result shows that it is indeed necessary, in a very strong sense: 

(3) We prove in \cref{prop:minimax_risk_bayes_net}, a lower bound on the sample complexity of proper learning in $\chi^2$ divergence, showing that \emph{any} learning algorithm whose $\chi^2$ guarantees holds on the whole domain must have sample complexity $\Omega(2^{n/2}/\eps)$, even for degree-$1$ Bayes nets.

\section{Related work}
\label{sec:related-works}
While learning graphical models in a range of settings and under various distance measures has a rich history, both in Statistics and Machine Learning, the corresponding task of \emph{testing} properties of an unknown distribution represented as a (succinct) graphical model has only been considered much more recently. \cite{DaskalakisD019} initiated the question of testing identity (goodness-of-fit) and independence of high-dimensional distributions with dependency structure modeled as an undirected graph (i.e., Ising models or, more generally, Markov Random Fields); this line of work was then continued in, e.g., ~\cite{DaskalakisDK17,GangradeNS18,NeykovL19,BezakovaBCSV20}, leading to a range of positive (algorithms), and negative (lower bounds) results. Very recently,~\cite{ChooDDK22} introduced the question of goodness-of-fit testing for latent Ising models, where only the leaf nodes of the tree are observable.

Focusing on another widely-studied type of graphical models, the concurrent works \cite{canonne2017testing}, and \cite{DaskalakisP17} studied analogous testing questions for Bayesian networks, where the underlying dependency structure is modeled as a directed graph. \cite{AcharyaBDK18} focused on the related tasks for \emph{causal} Bayesian networks, given the ability to perform interventions on the network. Finally, closest to our own work, \cite{DBLP:journals/corr/abs-2204-08690} studies the question of \emph{independence testing} for Bayesian networks, obtaining near-optimal sample complexity bounds for the task of deciding if a given high-dimensional distribution, promised to be a sparse Bayesian network, is in fact a product distribution.\smallskip

The other contribution of our work, high-probability learning in $\chi^2$ divergence for arbitrary discrete distributions, follows a long line of results related to density estimation under various distance measures (see, e.g.,~\cite{DBLP:conf/colt/KamathOPS15,DevroyeL01,Diakonikolas16}, as well as~\cite{Canonne:NoteLearningDistributions}, and references within). While learning arbitrary distributions under $\chi^2$ distance \emph{with constant success probability} has long been well understood, even to the optimal leading constant in the minimax estimation rate~\cite{DBLP:conf/colt/KamathOPS15}, obtaining high-probability bounds has proven quite elusive. Even for the weaker notion of learning under Kullback--Leibler (KL) divergence, such a high-probability learning result was only obtained very recently~\cite{DBLP:conf/stoc/0001GPV21}. 

Finally, to the best of our knowledge no result was known for learning Bayesian networks under $\chi^2$ divergence, even for constant success probability, besides the trivial bound one gets by treating the Bayesian network as an unstructured probability distribution over $\Sigma^n$. This is again to contrast with the case of KL divergence, for which optimal constant-probability learning bounds, \emph{and} high-probability learning bounds are known (cf. \cite{DBLP:conf/stoc/0001GPV21}, and references within).
\section{Overview}
Our algorithm follows the testing-by-learning framework developed in~\cite{DBLP:conf/nips/AcharyaDK15}, and since used in several works (e.g., \cite{DBLP:conf/soda/DaskalakisKW18,CDKL22}). Specifically, if the property one wants to test in TV distance is relatively easy to learn in a ``harder'' notion of distance, e.g., $\chi^2$, then it is possible to build an efficient, and possibly \emph{sample-optimal} tester by first learning the distribution in ${\chi^2}$, \emph{assuming} that it has the given property, and then use a ``$\chi^2$-$\tmop{TV}$ tolerant tester'' \cite{DBLP:conf/nips/AcharyaDK15} to test whether the hypothesis output by the learning algorithm is close to the actual distribution. The key is that these tolerant testers must not only reject distributions far in TV, but also accept those sufficiently close in $\chi^2$ (``tolerance''): now, if the unknown distribution $P$ does have the property, then the learning algorithm works as intended, and its output $\hat{P}$ is close to $P$: so, the tester will accept. However, if $P$ is far from the property, then either the learning algorithm fails and $\hat{P}$ is far from $P$ (and the tester rejects); or it still succeeds, but by the triangle inequality $\hat{P}$ must be itself far from the property (and this can be checked and detected, as we now have an explicit description of $\hat{P}$).

The key here is that the sample complexity of this ``$\chi^2$-tolerant'' tester is $O(\sqrt{|\Sigma|}/\varepsilon^2)$ for distributions over domain $\Sigma$~--~which is optimal (up to constants) for many testing tasks, and matches the sample complexity of the ``non-tolerant'' testers. Thus, as long as the learning stage can be done with much fewer than $O(\sqrt{|\Sigma|}/\varepsilon^2)$ samples, the overall approach yields a sample-efficient testing algorithm. (Moreover, this approach can be extended in many ways, e.g., for testing in Hellinger distance instead of TV~\cite{DBLP:conf/soda/DaskalakisKW18,DBLP:journals/corr/abs-2204-08690}.)

While the testing-by-learning idea seems relatively straightforward, the main technical contribution of this paper is in obtaining the required learning algorithm to apply it: namely, an efficient learning algorithm with high probability, and proper learning of Bayes nets (both in $\chi^2$). Indeed, and quite surprisingly, 
before our work it was still unclear whether one could learn a discrete
distribution in $\chi^2$ divergence with failure probability at most
$\delta$ by paying only an $O (\log (1 / \delta))$ dependence in the sample complexity. In particular, the ``obvious'' approaches based on applying either McDiarmid's inequality, or some sort of ``median trick'' fail for $\chi^2$ divergence, due respectively to the high sensitivity of the estimators, and the failure of the triangle inequality. We remark that achieving an \emph{exponentially} worse $O (1 / \delta)$ dependence is straightforward via Markov's
inequality {\cite{DBLP:conf/colt/KamathOPS15}}; however, this cost becomes impractical in settings when one requires an exponentially small failure probability, such as ours~--~as we need $2^d \cdot
n$ conditional distributions to be learned well simultaneously, and thus need to do a union bound over these many runs of a learning algorithm. 

Interestingly, learning with
high probability with only a $O (\log (1 / \delta))$ dependence in the sample complexity, and learning a Bayes net
with $\tilde{O} (2^d \cdot n / \eps)$ are both achievable for the
relatively easier $d_{\tmop{KL}}$ divergence (and even then, the high-probability result was only established recently~\cite{DBLP:conf/stoc/0001GPV21}). Unfortunately, learning in $\tmop{KL}$ is a much weaker guarantee, and trying to instantiate the aforementioned ``testing-by-learning'' framework with $\tmop{KL}$-$\tmop{TV}$ instead of $\chi^2$-TV would only yield a much looser $\tilde{O} (2^n / (n \cdot \eps^2))$
testing upper bound.

Below, we give a series of results on learning with respect to $\chi^2$ divergence: (1)~high-probability
learning with $O (| \Sigma | \cdot \log (1 / \delta) / \eps)$ sample; (2)~a near-proper Bayes net learning algorithm with sample complexity $\tilde{O}
(2^d n^2 / \eps)$; and (3)~an
exponential sample complexity lower bound of $\Omega (2^{n / 2} / \eps)$ for learning
Bayes nets. Later in \cref{sec:degree_tester}, we will build on the second result as the first step of our main result, the maximum in-degree testing algorithm of \autoref{theorem:informal:degree_testing}.\smallskip

\noindent\textbf{Preliminaries.} We use the standard asymptotic notations $O (\cdot), \Omega(\cdot), o (\cdot), \Theta (\cdot)$, and the (semi)-standard $\tilde{O}(\cdot)$, $\tilde{\Omega}(\cdot)$, and $\tilde{\Theta}(\cdot)$ to hide polylogarithmic factors in the argument.  Since we are focusing entirely on the in-degree of Bayes nets, all references to \emph{degree} will be implicitly in-degree unless stated otherwise. We use $A \lesssim B$ to indicate that $A \leqslant C
    \cdot B$ for some absolute constant $C$.

For a multivariate random variable $X$ supported on $\{0,1\}^n$, we use $X_i$ to denote its $i{\text{-th}}$ component (coordinate); for a Bayes net, we will use $\Pi_i$ to denote the set of parents of $X_i$.

\section{Learning in $\chi^2$}
\subsection{High probability learning in $\chi^2$}
Our analysis will largely follow the analysis of {\cite[Theorem 6.1,
and Claim 4.4]{DBLP:conf/stoc/0001GPV21}} for KL divergence, with a few crucial differences which allow us to extend it to $\chi^2$ divergence, and to obtain a strictly better sample complexity than prior work. Specifically, we go beyond the standard add-$1$ Laplace
estimator, and instead analyze the more general add-$K$ estimator for a suitable, \emph{non-constant} value of $K$. Recall that the add-$K$ estimator, given $N$ i.i.d.\ samples from some probability distribution over a domain $\Sigma$ (with empirical counts $N_1,\dots, N_{|\Sigma|}$), is defined by
\begin{equation}
    \label{eq:add-K}
    \hat{P}(i) = \frac{N_i + K}{N+K|\Sigma|}, \qquad i \in \Sigma\,.
\end{equation}
Intuitively, the parameter $K$ controls the amount of smoothing for our estimator. With $\chi^2$ divergence being a very stringent notion of distance (much more so than KL divergence, let alone TV distance), the idea to achieve high-probability learning guarantees is to increase the smoothing in order to counteract the risk of ``low-probability but catastrophic'' events which could make the  $\chi^2$ divergence blow up. We are then able to show that setting $K = \Theta (\log(1/\delta))$ achieves the desired sample complexity in the high-probability regime, much better than the Laplace estimator (which corresponds to $K=1$).
\begin{proposition}\label{prop:chi_square_learn}
  Fix any $\delta\in(0,1]$, and let $K=\Theta(\log(1/\delta))$. Given $N$ i.i.d.\ samples from an unknown probability distribution $P$ over an alphabet $\Sigma$, with probability at least $ 1-\delta$, the add-$K$ estimator yields a hypothesis $\hat{P}$ such that
  \[ d_{\chi^2} (P, \hat{P}) \lesssim \frac{| \Sigma | \log(| \Sigma |/\delta)}{N}.\]
  In particular, $N = O\left( \frac{| \Sigma |}{\varepsilon} \log   \frac{| \Sigma |}{\delta} \right)$ samples suffice to learn $P$ to $\chi^2$ divergence $\eps$ with probability $1-\delta$.
\end{proposition}
In contrast, using a similar analysis for the Laplace estimator would only yield a worse sample complexity of $O\!\left((|
  \Sigma |/\eps) \log^2(| \Sigma |/\delta) \right)$, off by a logarithmic factor in $\frac{|\Sigma|}{\delta}$. 
We will use this result in \cref{subsection:near-proper-learning-bayes-nets}.
  \begin{proof}[Analysis of the add-$K$ estimator]
  \label{proof:prop:chi_square_high_prob}
    Let $Q$ be the output of the add-$K$ estimator with $N$ i.i.d. samples from $P$. By the
    proof of {\cite[Claim 4.4]{DBLP:conf/stoc/0001GPV21}}, we have the following for each
    $i$, and $T_i$ being the count of the element $i$ from $N$ samples, with probability at least $1 - \delta$,
    \[ \left| P_i - \frac{T_i}{N} \right| \leqslant \sqrt{\frac{3 P_i \log \left(
      \frac{2}{\delta} \right)}{N}} + \frac{3 \log \left( \frac{2}{\delta}
      \right)}{N}. \]
    In the following, we will use $A \lesssim B$ to indicate that $A \leqslant C
    \cdot B$ for some absolute constant $C$.
\begin{eqnarray*}
  |P_i - Q_i | & \leqslant & \left| P_i - \frac{T_i}{N} \right| + \left|
  \frac{T_i}{N} - Q_i \right|\\
  & = & \left| P_i - \frac{T_i}{N} \right| + \left| \frac{T_i}{N} - \frac{T_i
  + K}{N + K| \Sigma |} \right|\\
  & = & \left| P_i - \frac{T_i}{N} \right| + \left| \frac{T_i | \Sigma | / N
  - 1}{N + K| \Sigma |} \right| K\\
  & \leqslant & \sqrt{\frac{3 P_i \log \left( \frac{2}{\delta} \right)}{N}} +
  \frac{3 \log \left( \frac{2}{\delta} \right)}{N} + \frac{T_i | \Sigma | / N}{N + K | \Sigma |} K + \frac{K}{N + K |
  \Sigma |}\\
  & \leqslant & \sqrt{\frac{3 P_i \log \left( \frac{2}{\delta} \right)}{N}}
  + \frac{3 \log \left( \frac{2}{\delta}
  \right)}{N} + \frac{K | \Sigma |}{N
  + K | \Sigma |}  \frac{T_i}{N}
  + \frac{K}{N}\\
  & \leqslant & \sqrt{\frac{3 P_i \log \left( \frac{2}{\delta} \right)}{N}} +
  \frac{3 \log \left( \frac{2}{\delta} \right)}{N} + \frac{K}{N}\\
  &  & + \frac{K | \Sigma |}{N + K | \Sigma |}  \left( \sqrt{\frac{3 P_i \log
  \left( \frac{2}{\delta} \right)}{N}} + 
  \frac{3 \log \left( \frac{2}{\delta} \right)}{N} +
  P_i \right)\\
  & \leqslant & 2 \sqrt{\frac{3 P_i \log \left( \frac{2}{\delta} \right)}{N}}
  + \frac{6 \log \left( \frac{2}{\delta}
  \right)}{N} + \frac{K | \Sigma |}{N
  + K | \Sigma |} P_i +
  \frac{K}{N} .
  \end{eqnarray*}
    We wish to show that $\Pr \left[ d_{\chi^2} (P, Q)\! =\! \sum_i \frac{(P_i -
    Q_i)^2}{Q_i} \leqslant C \varepsilon^2 \right] \geqslant 1 - \delta$, when
    taking at most $\frac{| \Sigma |}{\varepsilon^2} \log (| \Sigma | \delta^{- 1})$
    samples.
    We split the analysis into two cases:
    
    If $P_i \leqslant \frac{C' \log \left( \frac{2}{\delta} \right)}{N}$, then
    $\sqrt{\frac{3 P_i \log \left( \frac{2}{\delta} \right)}{N}} \leqslant
    \frac{\sqrt{3 C'} \log \left( \frac{2}{\delta} \right)}{N}$ and $\frac{K |
    \Sigma |}{N + K | \Sigma |} P_i \leqslant \frac{C' \log \left(
    \frac{2}{\delta} \right)}{N}$ and thus $2 \sqrt{\frac{3 P_i \log \left(
    \frac{2}{\delta} \right)}{N}} + \frac{6 \log \left(
    \frac{2}{\delta} \right)}{N} + \frac{K |
    \Sigma |}{N + K | \Sigma |} P_i
    + \frac{K}{N} \leqslant \left( 2 \sqrt{3 C'}
    + 6 + C' \right)  \frac{\log \left( \frac{2}{\delta} \right)}{N}
     +  \frac{K}{N}$. We then have
    \[ \frac{(P_i - Q_i)^2}{Q_i} \lesssim \frac{\left( \log \left(
       \frac{1}{\delta} \right) \right)^2}{N^2 Q_i} + \frac{K^2}{N^2 Q_i} \lesssim
       \frac{\left( \log \left( \frac{1}{\delta} \right) \right)^2}{NK} +
       \frac{K}{N}, \]
    since $Q_i \geqslant \frac{K}{N + K | \Sigma |} \geqslant \frac{K}{2 N}$, for
    $N \geqslant C' K \cdot | \Sigma |$.\footnote{If $N < C' K | \Sigma |$, then
    $Q_i > \frac{1}{(1 + C') | \Sigma |}$ and we can easily construct counter
    example for which $|P_i - Q_i | > \varepsilon$ for small enough
    $\varepsilon$.}
    
    If $P_{i} > \frac{C' \log \left( \frac{1}{\delta} \right)}{N}$,
    $| P_i - Q_i |
    \lesssim P_i  \left( \frac{1}{\sqrt{C'}} + \frac{1}{C'} +
    \frac{1}{\frac{N}{K | \Sigma |} + 1} \right) \lesssim P_i  \left(
    \frac{1}{\sqrt{C'}} + \frac{1}{C'} + \frac{1}{C' + 1} \right)$ and thus, for
    large enough $C'$, $Q_i \geqslant P_i / 2$, giving us
    \begin{eqnarray*}
      \frac{(P_i - Q_i)^2}{Q_i} & \lesssim & \frac{\left( \sqrt{\frac{P_i \log
      \left( \frac{1}{\delta} \right)}{N}} + \frac{K | \Sigma |}{N + K | \Sigma
      |} P_i + \frac{K}{N} \right)^2}{P_i}\\
      & \lesssim & \frac{\log \left( \frac{1}{\delta} \right)}{N} + \left(
      \frac{K | \Sigma |}{N + K | \Sigma |} \right)^2 P_i + \frac{K^2}{N^2
      P_i}\\
      & \lesssim & \frac{\log \left( \frac{1}{\delta} \right)}{N} + \left(
      \frac{K | \Sigma |}{N + K | \Sigma |} \right)^2 P_i + \frac{K^2}{N \log
      \frac{1}{\delta}} .
    \end{eqnarray*}
    Combining both cases, and applying a union bound, we have that with probability at
    least $1 - | \Sigma | \delta$,
    \begin{align}
      \sum_i \frac{(P_i \!-\! Q_i)^2}{Q_i} & \lesssim | \Sigma |  \left(\! \frac{\left(
      \log \left( \frac{1}{\delta} \right) \right)^2}{N K} \!+\! \frac{K}{N} \!+\!
      \frac{\log \left( \frac{1}{\delta} \right)}{N} \!+\! \frac{K^2}{N \log
      \frac{1}{\delta}} \!\right) \!+\! \left( \frac{K | \Sigma |}{N + K | \Sigma |} \right)^2 . \nonumber
    \end{align}
    Analyzing all the individual terms to have the summation to be bounded by
    $\varepsilon^2$, we need $N \geqslant \frac{| \Sigma |}{\varepsilon^2} \cdot
    \max \left\{ \frac{\log^2 (1 / \delta)}{K}, K, \log (1 / \delta),
    \frac{K^2}{\log (1 / \delta)}, \varepsilon \cdot K \right\} \geqslant \frac{|
    \Sigma |}{\varepsilon^2} \cdot \max \left\{ \frac{\log^2 (1 / \delta)}{K}, K,
    \log (1 / \delta), \frac{K^2}{\log (1 / \delta)} \right\} \geqslant \frac{|
    \Sigma | \log (1 / \delta)}{\varepsilon^2}$; and the optimal is reached when
    $K = \Theta (\log (1 / \delta))$. Rescaling $\delta$ to adjust for the $|
    \Sigma | \delta$ union bound, we conclude our proof.
  \end{proof}

\subsection{A near-proper learning algorithm for Bayes nets}
\label{subsection:near-proper-learning-bayes-nets}
For learning general distributions on the Boolean hypercube $\{0, 1\}^n$, it is
known that $\Theta (2^n / \eps^2)$ samples are both necessary and sufficient to
learn within $\chi^2$ divergence $\eps^2$ (with constant probability). As a direct consequence, we can
obtain a non-proper Bayes net learning algorithm on the entire support,
costing at most $O (2^n / \varepsilon^2)$ samples. However, this approach is not interesting in our context: it costs more samples to learn then to test, the resulting distribution is not a Bayes net, and the lack of triangle inequality in $\chi^2$ means that we cannot find a close enough distribution in the space of Bayes net to make it proper. In fact, it is unclear (to us) whether properly learning a Bayes net in $\chi^2$ is feasible in $O(2^n / \eps^2)$.

While one would hope that since Bayes nets have a sparse
description ($2^d \cdot n$ vs. $2^n$), it would allow us to bypass the $\Omega (2^n / \eps^2)$ lower
bound presented in {\cite{DBLP:conf/colt/KamathOPS15}} and possibly gives us a
much better upper bound. However, we will show in \cref{section:lower_bound_learning_bayes_nets} that, in stark contrast to
learning in $\tmop{KL}$, learning a sparse Bayes net in $\chi^2$ remains exponentially hard in sample complexity.

Loosely speaking, our lower bound hides a ``distinguished node'' behind a very biased common parent. By construction, the learning algorithm will never observe this distinguished ``rare'' node unless it takes exponentially many samples, which leads to an estimation error, entirely due to this very biased, large out-degree parent node. While this estimation error would be acceptable under TV distance or even KL divergence, the brittleness of $\chi^2$ divergence to low-probability elements leads to a very large estimation error overall. Thus, any $\chi^2$ learning algorithm \emph{on the entire support} cannot afford to be inaccurate even on such rare nodes, and thus must take exponentially many (in $n$) samples.

Nevertheless, given that testing is our end goal here, we only need a
majority of the support to be learned well to proceed, which allows us to bypass this lower bound. Specifically, for
testing in $\tmop{TV}$ distance, it suffices to guarantee that $\hat{P}$ is
close to $P$ in $\chi^2$, on a majority of the support $S$: \[d_{\chi^2} (P_S,
\hat{P}_S) = d_{\chi^2} (P,
\hat{P}, S) = \sum_{i \in S} \frac{(P_i - \hat{P}_i)^2}{\hat{P}_i},\] where
$\sum_{i \in S} P_i = P (S) \geqslant 1 - O (\varepsilon)$. For Hellinger, we
need something slightly stronger, but in the same spirit. Such framework is
already present
{\cite{DBLP:conf/nips/AcharyaDK15,DBLP:conf/soda/DaskalakisKW18}}, though it
is only implied in the analysis of {\cite{DBLP:conf/soda/DaskalakisKW18}} for
Hellinger distance. Thus, the main problem that remains is the availability of
such (sample-efficient) near-proper learning of Bayes nets algorithm.

Connecting back to the degree-$1$ lower bound, the main difficulty is the information bottleneck presented by the biased parent. By relaxing the problem into near-proper learning, we can simply give up on learning the rare set of parents altogether (when it is sufficiently small), since the sum of these masses will not exceed $O(\varepsilon)$. From here, a natural idea is to exclude a subset of the support, where a child and its parents' masses are at most $O(\varepsilon/(2^d \cdot n))$.\footnote{There are at most $2^d \cdot n$ parent configurations; excluding them with this threshold, we can still have a mass of $1 - O(\varepsilon)$ left.} 
This guarantee would be enough for our main result, testing with respect to TV distance; however, to extend it to the more stringent Hellinger testing guarantee, one needs to strengthen this to removing a subset of mass at most $O(\varepsilon^2)$, instead of $O(\varepsilon)$. As this extension changes little of the proof and provides a stronger result, in the remainder of the analysis, we focus on this goal.

In the interest of space and formatting, we use the abbreviated notation below in this section:
\begin{equation}
P (x_i, \pi_i (x)) \assign P_{X_i, \Pi_i} (x_i, \pi_i (x)) . 
\end{equation}
We define a support of interest to learn on:
\begin{equation}
S_k' \assign \left\{ x \in \{0, 1\}^k \mid \forall i \in [k], P (x_i, \pi_i
  (x)) \geqslant 4c \frac{\varepsilon^2}{2^{d + 1} n} \right\} .
  \label{eq:def:S_k_prime}
\end{equation}
We also define a superset $S_k$ of $S'_k$:
\begin{equation}
    S_k \assign \left\{ x \in \{0, 1\}^k \mid \forall i \in [k], P
  (x_i, \pi_i (x)) \geqslant c \frac{\varepsilon^2}{2^{d + 1} n} \right\}.
  \label{eq:def:S_k}
\end{equation}
In what follows, we assume the algorithm is provided with a set $\tilde{S}_k$ such that $S_k' \subseteq \tilde{S}_k \subseteq S_k$. Indeed, \cref{lemma:tilde_S_k_algo_analysis}, analyzed in \cref{section:estimation_of_support} via \autoref{algo:majority_set_learning}, guarantees that we can efficiently learn such a set with high probability.
\begin{algorithm2e}[h]
  \SetKwComment{Comment}{/* }{ */}
  \SetKwInOut{Input}{Input}
  \DontPrintSemicolon
  \Input{Sample access to distribution $P$, accuracy parameter $\varepsilon$ and a DAG $G$.}
  Draw a multiset $S$ of $m$ samples from $P$, where $m = 3 \cdot 2^{d + 1} n \log (6 \cdot 2^{d + 1} n) / (c \cdot \varepsilon^2)$.\;
  Let $N_{x_i, \pi_i}$ be the number of occurrences of $(x_i, \pi_i)$ in $S$.\;
  Let $Z_{x_i, \pi_i} = \frac{N_{x_i, \pi_i}}{m}$ and $\bar{S} \assign \varnothing$.\;
  \For{$i = 1;\ i \leqslant n;\ i = i + 1$}{
    \For{$(x_i, \pi_i) \in \{0, 1\}^{|\Pi_i|+1}$} {
      \lIf{$Z_{x_i, \pi_i} \leqslant 2 c \frac{\varepsilon^2}{2^{d + 1} n}$}
      {    
        $\bar{S} \leftarrow \bar{S} \cup \{ i, (x_i, \pi_i) \}$
      }
    }
  }
  Mark all pairs in $\bar{S}:X_i=x_i, \Pi_i=\pi_i$ as excluded (from $\{0,1\}^n$) and return the remaining support.
  \caption{$\varepsilon^2$-majority support identification  \label{algo:majority_set_learning}}
\end{algorithm2e}

A straightforward way to approach this problem is to consider learning
guarantees on all the conditionals, i.e., $d_{\chi^2} (P_{X_i | \Pi_i}, Q_{X_i
| \Pi_i}) \leqslant \frac{\varepsilon^2}{n}$ for any $\Pi_i = \pi_i$ and hence\footnote{While $ 1 + d_{\chi^2} (P, Q) $ should be $2 P(S) - Q(S) + d_{\chi^2} (P_S, Q_S)$, it does not affect the analysis if $P(S) \geqslant 1 - O(\varepsilon^2)$.} \[1 + d_{\chi^2} (P, Q) \leqslant \prod_{i = 1}^n (1 + \max_{\pi_i}
d_{\chi^2} (P_{X_i | \Pi_i = \pi_i}, Q_{X_i | \Pi_i = \pi_i})) \leqslant \left( 1 + \frac{\varepsilon^2}{n} \right)^n \leqslant 1+ \varepsilon^2.\] Coupled this with the fact that mass on the parents is lower bounded by $\frac{\varepsilon}{2^d n}$, we can obtain enough samples for learning each conditional by paying an extra $O \left( \frac{2^d n}{\varepsilon} \right)$ per $O \left( \frac{n}{\varepsilon^2} \right )$ (and some log factors for high probability learning and a union bound), giving us a sample complexity of $\tilde{O} \left( \frac{2^d
n^2}{\varepsilon^3} \right)$ and in the case of $1 - O (\varepsilon^2)$, a
sample complexity of $\tilde{O} \left( \frac{2^d n^2}{\varepsilon^4} \right)$. As we will see later, we can do something slightly better to tighten the dependence on $\varepsilon$.

\begin{algorithm2e}[h]
  \SetKwComment{Comment}{/* }{ */}
  \SetKwInOut{Input}{Input}
  \DontPrintSemicolon
  \Input{Sample access to distribution $P$, accuracy parameter $\varepsilon$.}
  \Comment {Obtain the majority support $\mathcal{A}$.}
  $\mathcal{A} \leftarrow$ call \cref{algo:majority_set_learning} with $P$ and $\varepsilon$.\;
  Draw a multiset $S$ of $m$ samples from $P$, where $m=O ( 2^d n^2 \log(2^d n)/\varepsilon ^2 )$.\;
  $K \leftarrow \Theta(\log (2^{d+1} \cdot n))$\;
  Let $N_{x_i, \pi_i}$ be the count of $X_i=x_i, \Pi_i=\pi_i$ out of the $m$ samples.\;
  \For{$i = 1;\ i \leqslant n;\ i = i + 1$}{
    \For{$x_i, \pi_i \in \{0, 1\}^{|\Pi_i|+1}$} {
      $Q(x_i|\pi_i) \leftarrow \frac{K + N_{x_i, \pi_i}}{\sum_{x'_i \in \{0,1\}} K + N_{x'_i, \pi_i}}$\;
    }
  }
  \Comment {A mass shifting step in \eqref{eq:mass_shifting_Hellinger} is necessary for testing in Hellinger downstream.}
  $Q(x_1, \ldots, x_n) \leftarrow \prod_{i=1}^n Q(x_i|\pi_i)$
  \caption{$\varepsilon^2$-near proper learning in $\chi^2$ \label{algo:near_proper_learning_chi_square}}
\end{algorithm2e}

We will defer the proof of
\cref{lemma:near_proper_chi_squared_learning_bayes_recur} to \cref{sec:high_prob_events,proof:lemma:near_proper_chi_squared_learning_bayes_recur}, and provide a proof sketch in-place.

\begin{lemma}
  \label{lemma:near_proper_chi_squared_learning_bayes_recur}
  When $m = O \left(\frac{2^d n^2 \log (2^d n)}{c \varepsilon^2}
  \right)$, the hypothesis $Q$ output by \autoref{algo:near_proper_learning_chi_square} satisfies the following recurrence:  
  \[ d_{\chi^2} (P_{X_1, \ldots, X_k}, Q_{X_1, \ldots, X_k}, S_k) \leqslant
   \left( 1 + \frac{1}{n} \right) d_{\chi^2} (P_{X_1, \ldots, X_{k - 1}},
   Q_{X_1, \ldots, X_{k - 1}}, S_{k - 1}) + O \left( \frac{\varepsilon^2}{n}
   \right) , \]
  which, by recursion, implies
  \[ d_{\chi^2} (P_{X_1, \ldots, X_n}, Q_{X_1, \ldots, X_n}, S_n) \leqslant O
     (\varepsilon^2) . \]
\end{lemma}

\begin{proof}[Sketch]
  By lower bounding all parents' masses, we can guarantee that the $\chi^2$ error on all conditionals of the hypothesis will be bounded by,
  \[d_{\chi^2} (P_{X_i | \Pi_i}, Q_{X_i | \Pi_i}) \leqslant \frac{O(1)}{m \cdot P(\Pi_i)}\]
  with an application of a Chernoff bound; in contrast, the naive approach is to simply take $P(\Pi_i) \geqslant \frac{\varepsilon^2}{2^d n}$.
  With this and some careful rearrangement of the residual terms, we are able to show this recurrence. Such arrangement is tricky because the standard $\chi^2$ form becomes $-2 P(S) + Q(S) + \sum_i \frac{P^2_i}{Q_i}$ rather than the standard $-1 + \sum_i P_i^2 / Q_i$ form.
\end{proof}
For testing in $d_H$, we need the following tweaks: from the obtained $Q$, we shift masses to get a $\tilde{Q}$
s.t., $\tilde{Q} (\tilde{S}_n) \geqslant 1 - O (\varepsilon^2)$, while
maintaining the graphical structure of $Q$ and the $d_{\chi^2}$ closeness from
$P$ on $\tilde{S}_n$; though this could potentially make the $d_{\chi^2}$ on the entire support
unbounded (it will be infinity since there are $0$'s in the denominator), it does not affect our downstream testing. As we will see later on
in the analysis, this is an extra (and seemingly necessary) step for testing maximum in-degree in
$d_H$; but for $d_{\tmop{TV}}$, interestingly, $Q$ and $\tilde{S}_n$ is sufficient.

\begin{thm}
  \label{theorem:near_proper_chi_squared_learning_bayes}Given $O \left(2^d n^2 \log(2^d n) / \varepsilon^2 \right)$ samples, we can obtain a
  {\tmem{proper}} hypothesis $Q$ of the degree-$d$ Bayes net $P$, and a $\tilde{S}_n
  \subset \{0, 1\}^n$ on which $d_{\chi^2} (P, Q, \tilde{S}_n) \leqslant O
  (\varepsilon^2)$, and $P (\tilde{S}_n) \geqslant 1 - O (\varepsilon^2)$.
  
  Additionally, with some post-processing (without extra samples), we can obtain another proper
  hypothesis $\tilde{Q}$ with guarantees subsuming the above, and:
  $\tilde{Q} (\tilde{S}_n) = 1 > 1 - O (\varepsilon^2)$.
\end{thm}

\begin{proof}
  By \cref{lemma:near_proper_chi_squared_learning_bayes_recur}, and the
  fact that $\tilde{S}_n \subseteq S_n$,
  \begin{align*}
    d_{\chi^2} (P_{X_1, \ldots, X_n}, Q_{X_1, \ldots, X_n}, \tilde{S}_n) 
    \!\leqslant\! d_{\chi^2} (P_{X_1, \ldots, X_n}, Q_{X_1, \ldots, X_n}, S_n),
  \end{align*}
  which is $O (\varepsilon^2)$. This concludes the proof.
\end{proof}

\subsection{Efficient Estimation of $O (\varepsilon^2)$-effective support}
\label{section:estimation_of_support}
We write the formal statement regarding $\tilde{S}_n$ assumed in \cref{subsection:near-proper-learning-bayes-nets} as \cref{lemma:tilde_S_k_algo_analysis}.
\begin{lemma}
  \label{lemma:tilde_S_k_algo_analysis}There exists a routine (\cref{algo:majority_set_learning}) that takes at
  most $O \left( 2^{d + 1} n \log (2^{d + 1} n) /
  (c \cdot \varepsilon^2) \right)$ samples from $P$, and return an approximation
  $\tilde{S}_n$ with the following guarantees: for all $k \in [n]$, and with
  probability at least $5 / 6$,
  \[ S_k' \subseteq \tilde{S}_k \subseteq S_k, \]
  where $S_k', S_k$ are as defined in \cref{eq:def:S_k_prime,eq:def:S_k}.
\end{lemma}

    \begin{proof}
    \label{proof:lemma:tilde_S_k_algo_analysis}
      We prove by analyzing \autoref{algo:majority_set_learning}. The algorithm takes $m = \frac{3}{c} \cdot 2^{d + 1} n \log (6 \cdot 2^{d + 1} n) /
      \varepsilon^2$ samples and checks if the ratio of each $X_i, \Pi_i$ exceeds $2
      c \frac{\varepsilon^2}{2^{d + 1} n}$ -- for $x \in \{0, 1\}^n$, if all $i \in
      [n]$, $(x_i, \pi_i) \assign (x_i (x), \pi_i (x)) \in \{0, 1\}^{| \Pi_i | +
      1}$, $N_{x_i, \pi_i} \geqslant 2 c \frac{\varepsilon^2 \cdot m}{2^{d + 1} n}$
      then $x$ gets added to $\tilde{S}_n$, where $N_{x_i, \pi_i}$ is the number of
      occurrences of $X_i = x_i, \Pi_i = \pi_i$ over the $m$ samples.
    
      Our argument here is to ensure that the $a \in \{0, 1\}^{| \Pi_i | + 1}$
      with smaller masses than $c \frac{\varepsilon^2}{2^{d + 1} n}$ won't pass
      the procedure, and that for each $i \in [n]$, at most $O (\varepsilon^2 / n)$
      of masses are dropped. We do so via a Chernoff bound, a stochastic dominance argument
      and a union bound. First, consider the Bernoulli distribution $T \sim
      \tmop{Bern} (m, p)$, where $p = c \frac{\varepsilon^2}{2^{d + 1} n}$. By Chernoff's inequality,
      
      \[ \Pr \left[ T \geqslant 2 \cdot c \frac{\varepsilon^2}{2^{d + 1} n} \right]
       = \Pr [T \geqslant 2 mp] \leqslant \exp (- mp / 3) = \exp \left( - mc
       \frac{\varepsilon^2}{2^{d + 1} n} / 3 \right) = \frac{1}{6 \cdot 2^{d + 1}
       n} . \]
      
      Since any $T' \sim \tmop{Bern} (m, p')$ is first-order stochastically
      dominated by $T \sim \tmop{Bern} (m, p)$ if $p' \leqslant p$, thus for
      $p' < \frac{c}{2}  \frac{\varepsilon^2}{2^{d + 1} n}$, $\Pr \left[ T'
      \geqslant 2 c \frac{\varepsilon^2}{2^{d + 1} n} \right] \leqslant \Pr \left[
      T \geqslant 2 c \frac{\varepsilon^2}{2^{d + 1} n} \right] \leqslant
      \frac{1}{2^{d + 1} n}$.
      
      Therefore, for $p \leqslant c \frac{\varepsilon^2}{2^{d + 1} n}, T \sim
      \tmop{Bern} (m, p)$, $\Pr \left[ T \geqslant 2 c \frac{\varepsilon^2}{2^{d +
      1} n} \right] \leqslant \frac{1}{2^{d + 1} n}$.
      
      Via similar argument, all $p \geqslant 4 c \frac{\varepsilon^2}{2^{d + 1}
      n}$ will pass the test with high probability: let $T'' \sim \tmop{Bern} (m,
      p'')$, where $p'' \geqslant 4 c \frac{\varepsilon^2}{2^{d + 1} n}$, and by
      Chernoff,

      \[\Pr \left[ T'' \leqslant \frac{1}{2} \cdot 4 c \frac{\varepsilon^2}{2^{d +
        1} n} \right] = \Pr \left[ T'' \leqslant \frac{1}{2} mp'' \right]
        \leqslant \exp (- mp'' / 8) 
        \leqslant \exp \left( - \frac{1}{2} mc \frac{\varepsilon^2}{2^{d + 1} n} \right)
        \leqslant \frac{1}{6 \cdot 2^{d + 1} n} .\]
        
      Finally, a union bound over all elements $a \in \{0, 1\}^{\Pi_i + 1} =
      \tmop{Support} (X_i, \Pi_i)$ implies that (w.h.p.) all $i \in [n]$ and $a
      \in \{0, 1\}^{\Pi_i + 1}$: $P_{X_i, \Pi_i} (a) \geqslant 4 c
      \frac{\varepsilon^2}{2^{d + 1} n}$ will pass the test; $P_{X_i, \Pi_i} (a)
      \leqslant c \frac{\varepsilon^2}{2^{d + 1} n}$ will fail the test. This
      tells us that, for $k \in [n]$, $\left\{ x \in \{0, 1\}^k | \forall i \in
      [k], P_{X_i, \Pi_i} (x_i, \pi_i (x)) \geqslant 4 c \frac{\varepsilon^2}{2^{d
      + 1} n} \right\} \subset \tilde{S}_n \subset \left\{ x \in \{0, 1\}^k |
      \forall i \in [k], P_{X_i, \Pi_i} (x_i, \pi_i (x)) \geqslant c
      \frac{\varepsilon^2}{2^{d + 1} n} \right\}$.
    \end{proof}

  \subsection{Some useful high probability events for \cref{lemma:near_proper_chi_squared_learning_bayes_recur}}\label{sec:high_prob_events}

In this subsection, we analyze some crucial high probability events in the
form of \cref{eq:nice_chi_squared_distance_not_shifted,eq:nice_chi_squared_distance_shifted}, to facilitate our main proof.
But first we need to build up some technical machinery.
We will use the term
{\tmem{configuration}} to refer to some set of binary strings that is a subset
of $\{0, 1\}^n$. First, we define the set of configurations for parent nodes
with ``large enough'' masses,
\[ A_{S_k} \assign \left\{ a \in \{0, 1\}^{| \Pi_k |} \mid P_{\Pi_k} (a) \geqslant
   c \cdot \frac{\varepsilon^2}{2^d n} \right\} . \]
We define another closely related set of configurations: let $C (a, S_k)$ be
the set of configurations (excluding the current parent nodes $\Pi_k$ and last node
$X_k$) remaining in $S_k$,\footnote{This term is used to analyze
(independently) the two terms generated by conditioning on the parent nodes
$\Pi_k$ in this Markov chain: $\{ X_1, \ldots, X_{k - 1} \} \backslash \Pi_k
\rightarrow \Pi_k \rightarrow X_k$.} given $\Pi_k = a$ and it is independent of the value $X_k$ takes. Formally:
\begin{align}
   \rightarrow\; & C(a, S_k) = \{ x \in \{0, 1\}^{k - 1 - | \Pi_k |} \mid  \exists x'
  \in \{0, 1\} \nobracket \nobracket \{(X_1, \ldots, X_{k - 1} \setminus \Pi_k) = x,
  \Pi_k = a, X_k = x' \} \in S_k \} \nonumber
\end{align}
Or equivalently $C (a, S_k) = \{ x \in \{0, 1\}^{k - 1 - | \Pi_k |} \mid$ $\{(X_1,
\ldots, X_{k - 1} \setminus \Pi_k) = x, \Pi_k = a\} \in S_{k - 1} \}$ --
fixing $\Pi_k$ to $a$ out of the $X_1, \ldots, X_k$ variables and checking if $x, a
\in S_{k - 1}$. Similarly, 
\begin{align}
  \rightarrow\; & B(a, S_k) = \{ x \in \{0, 1\}\mid  \forall x' \in \{0, 1\}^{k - 1
  - | \Pi_k |} \nobracket \nobracket \{(X_1, \ldots, X_{k - 1} \setminus \Pi_k) = x', \Pi_k = a,
  X_k = x \} \in S_k \} \nonumber
\end{align}

Or we may equivalently define $B (a, S_k) = \left\{ x \in \{ 0, 1 \} \mid P (\Pi_k = a, X_k = x)
\geqslant c \cdot \frac{\varepsilon^2}{2^{d + 1} n} \right\}$.

Other sets of configurations include: $C_k \assign \{0, 1\}^{k - 1 - |
\Pi_k |}$, the full set of configurations for $\{X_1, \ldots, X_{k - 1}
\setminus \Pi_k \}$ without any restrictions; just as $B_k \assign \{0,
1\}$ is related to $X_k$; $A_k \assign \{ 0, 1 \}^{| \Pi_k |}$; $A_{S^c_k} \assign A_k
\backslash A_{S_k}$; $ B (a, S^c_k) \assign B_k
\backslash B (a, S_k)$; $ C (a,
S^c_k) \assign C_k \backslash C (a, S_k)$. Note that $x$ where $(x_k, \pi_k) \in
A_{S_k}$ may not imply $x \in S_k$ (its converse is true); and knowing
$\Pi_k = a$, in order for $x \in S_k$, it has to satisfy both constraints --
one from $\{X_1, \ldots, X_{k - 1} \setminus \Pi_k \}$ and the other $\{X_k
\}$.

Let $m_{\Pi_k = a}$ be the random variable counting the number of samples
with $\Pi_k = a$. For any $x \in S_n$, every $ k \in [n]$, and $ a \in \{0, 1\}^{|
\Pi_k |}$ satisfies $P (\Pi_k = a) = \sum_{x \in X_k} P (x, \Pi_k = a)
\geqslant 2 \cdot c \frac{\varepsilon^2}{2^{d + 1} n} = c
\frac{\varepsilon^2}{2^d n}$, and thus by Chernoff, $m_{\Pi_k = a} \geqslant
\frac{1}{2} mP (\Pi_k = a)$, with v.h.p. We condition on this event, and with
the learning result in \cref{prop:chi_square_learn}, by setting $K=\log(6 \cdot 2^{d+1} n)$, we can derive: for all $a \in A_{S_k}$,
\begin{equation}
  \Pr \left[ - 1 + \sum_{b \in B} \frac{P_{X_k | \Pi_k}^2 (b|a)}{Q_{X_k |
  \Pi_k} (b|a)} \geqslant \frac{c' \log (2^{d + 1} n)}{m_{\Pi_{k = a}}} \right]
  \leqslant \frac{1}{6}  \frac{1}{2^{d + 1} n},
\end{equation}
for some constant $c'$.

And thus, with high probability, for all $a \in A_{S_k}$,
\begin{equation}
  - 1 + \sum_{b \in B} \frac{P_{X_k | \Pi_k}^2 (b|a)}{Q_{X_k | \Pi_k} (b|a)}
  \leqslant \frac{c' \log (2^{d + 1} n)}{m_{\Pi_{k = a}}} \leqslant \frac{2 c'
  \cdot \log (2^{d + 1} n)}{mP (\Pi_k = a)} .
  \label{eq:nice_chi_squared_on_all}
\end{equation}
We then again condition on this event happening; since \[\sum_{b \in B (a,
S_k)} \frac{P_{X_k | \Pi_k}^2 (b|a)}{Q_{X_k | \Pi_k} (b|a)} \leqslant \sum_{b
\in B} \frac{P_{X_k | \Pi_k}^2 (b|a)}{Q_{X_k | \Pi_k} (b|a)},\] giving us,
\begin{equation}
  - 1 + \sum_{b \in B (a, S_k)} \frac{P_{X_k | \Pi_k}^2 (b|a)}{Q_{X_k | \Pi_k}
  (b|a)} \leqslant \frac{2 c' \cdot \log (2^{d + 1} n)}{mP (\Pi_k = a)} .
  \label{eq:nice_chi_squared_distance_not_shifted}
\end{equation}
In the later part of the proof of
\cref{lemma:near_proper_chi_squared_learning_bayes_recur}, we will need a
stronger condition on our density estimate (a slightly different $Q$ by moving
at most $O (\varepsilon^2)$ mass around) -- we will call it $\tilde{Q}$, and we
can obtain it directly from $Q$ and $\tilde{S}_n$:
\begin{equation} \tilde{Q}_{X_k | \Pi_k} (b|a) \assign \left\{ \begin{array}{cl}
     \frac{Q_{X_k | \Pi_k} (b|a)}{\sum_{b \in B (a, \tilde{S}_k)} Q_{X_k |
     \Pi_k} (b|a)} & \text{ if } b \in B (a, \tilde{S}_k),\\
     0 & \text{ otherwise.}
   \end{array} \right.  \label{eq:mass_shifting_Hellinger} \end{equation}
Furthermore, $\tilde{Q}$ on $\tilde{S}_n$ shares very similar guarantees as
$Q$ on $S_n$. For every $ a \in A_{\tilde{S}_k} \subseteq A_{S_k}$, since $B (a,
\tilde{S}_k) \subseteq B (a, S_k)$ (conditioning on
\autoref{eq:nice_chi_squared_on_all}),

\begin{align}
  \sum_{b \in B (a, \tilde{S}_k)} \frac{P_{X_k | \Pi_k}^2 (b|a)}{Q_{X_k |
  \Pi_k} (b|a)} 
  \leqslant  \sum_{b \in B (a, S_k)} \frac{P_{X_k | \Pi_k}^2 (b|a)}{Q_{X_k |
  \Pi_k} (b|a)} \leqslant 1 + \frac{2 c' \cdot \log (2^{d + 1} n)}{mP (\Pi_k = a)} . \nonumber
\end{align}

As moving masses from $ B (a, \tilde{S}_k^c)$ to $ B (a,
\tilde{S}_k)$ will only reduce the quantity $\frac{P_{X_k | \Pi_k}^2
(b|a)}{\tilde{Q}_{X_k | \Pi_k} (b|a)}$ further, we have
\begin{equation}
  - 1 + \sum_{b \in B (a, \tilde{S}_k)} \frac{P_{X_k | \Pi_k}^2
  (b|a)}{\tilde{Q}_{X_k | \Pi_k} (b|a)} \leqslant \frac{2 c' \cdot \log (2^{d +
  1} n)}{mP (\Pi_k = a)} . \label{eq:nice_chi_squared_distance_shifted}
\end{equation}
Finally, by a union bound, we have that with probability at least $5 / 6$, for
all $k \in [n]$ and $\Pi_k = a$, both statements in
\cref{eq:nice_chi_squared_distance_not_shifted,eq:nice_chi_squared_distance_shifted} will hold. Throughout the rest of
the appendix, we will condition on these two statements.

\subsection{Proof of \cref{lemma:near_proper_chi_squared_learning_bayes_recur}}\label{proof:lemma:near_proper_chi_squared_learning_bayes_recur}

We write the partial sum of $\chi^2$ between $P_{X_1, \ldots, X_k}$ and
$Q_{X_1, \ldots, X_k}$ on the subset $S_k \subset \{0, 1\}^k$ as:

\begin{eqnarray*}
  d_{\chi^2} (P_{X_1, \ldots, X_k}, Q_{X_1, \ldots, X_k}, S_k) & = & \sum_{x
  \in S_k} \frac{(P_{X_1, \ldots, X_k} (x) - Q_{X_1, \ldots, X_k}
  (x))^2}{Q_{X_1, \ldots, X_k} (x)}\\
  & = & \sum_{x \in S_k} - 2 P_{X_1, \ldots, X_k} (x) + Q_{X_1, \ldots, X_k}
  (x) + \frac{P_{X_1, \ldots, X_k} (x)^2}{Q_{X_1, \ldots, X_k} (x)}\\
  & = & - 2 P_{X_1, \ldots, X_k} (S_k) + Q_{X_1, \ldots, X_k} (S_k) + \sum_{x
  \in S_k} \frac{P_{X_1, \ldots, X_k} (x)^2}{Q_{X_1, \ldots, X_k} (x)},
\end{eqnarray*}

By definition of $S_k$, we can write the sum over $x \in S_k$
as $a \in A_{S_k}, b \in B (a, S_k)$ and $g \in C (a, S_k)$:
\begin{eqnarray}
  &  & d_{\chi^2} (P_{X_1, \ldots, X_k}, Q_{X_1, \ldots, X_k}, S_k)
  \nonumber\\
  & = & \left\{ \sum_{a \in A_{S_k}} \frac{P_{\Pi_k}^2 (a)}{Q_{\Pi_k} (a)}
  \cdot 
  \sum_{g \in C (a, S_k)} \frac{P_{X_1, \ldots, X_{k - 1}
  \setminus \Pi_k | \Pi_k}^2 (g \mid a)}{Q_{X_1, \ldots, X_{k - 1} \setminus
  \Pi_k | \Pi_k} (g \mid a)} \cdot \left. \sum_{b \in B (a, S_k)} \frac{P_{X_k
  | \Pi_k}^2 (b \mid a)}{Q_{X_k | \Pi_k} (b \mid a)} \right\} \right.
  \nonumber\\
  &  & - (2 P_{X_1, \ldots, X_k} (S_k) - Q_{X_1, \ldots, X_k} (S_k))
  \nonumber\\
  & = & \left\{ \sum_{a \in A_{S_k}} \frac{P_{\Pi_k = a}^2}{Q_{\Pi_k = a}}
  \cdot 
  \sum_{g \in C (a, S_k)} \frac{P_{X_1, \ldots, X_{k - 1}
  \setminus \Pi_k | \Pi_k}^2 (g \mid a)}{Q_{X_1, \ldots, X_{k - 1} \setminus
  \Pi_k | \Pi_k} (g \mid a)} \cdot \left. \left( - 1 + \sum_{b \in B (a, S_k)}
  \frac{P_{X_k | \Pi_k}^2 (b \mid a)}{Q_{X_k | \Pi_k} (b \mid a)} \right)
  \right\} \right. \nonumber\\
  &  & - \left( 2 P_{X_1, \ldots, X_{k - 1}} (S_{k - 1}) - 
  Q_{X_1, \ldots, X_{k - 1}} (S_{k - 1}) \right) \nonumber\\
  &  & + \sum_{a \in A_{S_k}} 
  \hspace{-0.1em} \sum_{g \in C (a, S_k)} 
  \frac{P_{\Pi_k = a}^2}{Q_{\Pi_k = a}} \cdot \frac{P_{X_1, \ldots, X_{k - 1}
  \setminus \Pi_k | \Pi_k}^2 (g \mid a)}{Q_{X_1, \ldots, X_{k - 1} \setminus
  \Pi_k | \Pi_k} (g \mid a)} \nonumber\\
  &  & + (2 P_{X_1, \ldots, X_{k - 1}} (S_{k - 1}) - Q_{X_1, \ldots, X_{k -
  1}} (S_{k - 1})) \nonumber\\
  &  & - (2 P_{X_1, \ldots, X_k} (S_k) - Q_{X_1, \ldots, X_k} (S_k)) 
  \label{eq:mark_add_minus_terms}\\
  & = & \left\{ \sum_{a \in A_{S_k}} \frac{P_{\Pi_k}^2 (a)}{Q_{\Pi_k} (a)} 
  \sum_{g \in C (a, S_k)}
  \frac{P_{X_1, \ldots, X_{k - 1} \setminus \Pi_k | \Pi_k}^2 (g \mid
  a)}{Q_{X_1, \ldots, X_{k - 1} \setminus \Pi_k | \Pi_k} (g \mid a)} \cdot
  \left. \left( - 1 + \sum_{b \in B (a, S_k)} \frac{P_{X_k | \Pi_k}^2 (b \mid
  a)}{Q_{X_k | \Pi_k} (b \mid a)} \right) \right\} \right. \nonumber\\
  &  & + d_{\chi^2} (P_{X_1, \ldots, X_{k - 1}}, Q_{X_1, \ldots, X_{k - 1}},
  S_{k - 1}) \nonumber\\
  &  & + (2 P_{X_1, \ldots, X_{k - 1}} (S_{k - 1}) - Q_{X_1, \ldots, X_{k -
  1}} (S_{k - 1})) \nonumber\\
  &  & - (2 P_{X_1, \ldots, X_k} (S_k) - Q_{X_1, \ldots, X_k} (S_k)) 
  \label{eq:recurrence_to_k_minus_1}\\
  & \leqslant & 2 c \frac{\varepsilon^2}{n} + 4 c'  \frac{2^d \log (2^d
  n)}{m} + \left( 1 + \frac{2 c' 2^d n \log (2^d n)}{cm \varepsilon^2} \right)
  \times d_{\chi^2} (P_{X_1, \ldots, X_{k - 1}}, Q_{X_1, \ldots, X_{k - 1}},
  S_{k - 1}) .  \label{eq:apply_lemmas}
\end{eqnarray}

We obtain \eqref{eq:mark_add_minus_terms} by adding and subtracting the same terms, giving us a recurrence in $S_{k-1}$ in \eqref{eq:recurrence_to_k_minus_1}. From this, we get \eqref{eq:apply_lemmas} by applying a couple of technical results: \cref{lemma:residual_1,lemma:residual_2}. Finally, by setting $m=\frac{4 c 2^d n^2 \log(2^d n)}{\varepsilon^2}$, we get the desired recursive formulation 
for $d_{\chi^2} (P_{X_1, \ldots, X_k}, Q_{X_1, \ldots, X_k}, S_k)$, concluding our proof of \cref{lemma:near_proper_chi_squared_learning_bayes_recur}.

We now prove the aforementioned technical lemmata.

\begin{lemma}\label{lemma:residual_1}
  \begin{align}
    (2 P_{X_1, \ldots, X_{k - 1}} (S_{k - 1}) - Q_{X_1, \ldots, X_{k - 1}}
    (S_{k - 1}))  - 
    (2 P_{X_1, \ldots, X_k} (S_k) - Q_{X_1, \ldots, X_k} (S_k))  \leqslant 2 c \frac{\varepsilon^2}{n} . \nonumber
  \end{align}
\end{lemma}

\begin{proof}
    \begin{eqnarray*}
      &  & (2 P_{X_1, \ldots, X_{k - 1}} (S_{k - 1}) - Q_{X_1, \ldots, X_{k - 1}}
      (S_{k - 1})) - (2 P_{X_1, \ldots, X_k} (S_k) - Q_{X_1, \ldots, X_k} (S_k))\\
      & = & \sum_{x_1, \ldots, x_{k - 1} \in S_{k - 1}} \sum_{x_k \in \{0, 1\}} 2
      P_{X_1, \ldots, X_k} (x_1, \ldots, x_k) - \sum_{x_1, \ldots, x_k \in S_k} 2
      P_{X_1, \ldots, X_k} (x_1, \ldots, x_k)\\
      &  & - (Q_{X_1, \ldots, X_{k - 1}} (S_{k - 1}) - Q_{X_1, \ldots, X_k}
      (S_k))\\
      & = & \sum_{x_1, \ldots, x_{k - 1} \in S_{k - 1}, x_k \in S^c (x_1, \ldots,
      x_{k - 1})} (2 P (x_1, \ldots, x_k) - Q (x_1, \ldots, x_k))\\
      & \leqslant & \sum_{x_1, \ldots, x_{k - 1} \in \{0, 1\}^{k - 1}} \sum_{x_k
      \in S^c (x_1, \ldots, x_{k - 1})} 2 P (x_1, \ldots, x_k)\\
      & = & 2 \sum_{\{x_1, \ldots, x_{k - 1} \setminus \pi_k \} \in \{0, 1\}^{k -
      1 - | \Pi_k |}} \sum_{\pi_k} \sum_{x_k \in S^c (\pi_k)} 
         P (x_1, \ldots, x_{k - 1}
      \setminus \pi_k | \pi_k) P (x_k, \pi_k)\\
      & \leqslant & 2 \sum_{x_k, \pi_k \in S^c (X_k, \Pi_k)} P_{X_k, \Pi_k} (x_k,
      \pi_k) \leqslant 2 c \frac{\varepsilon^2}{n} .
    \end{eqnarray*}
  
\end{proof}

\begin{lemma}\label{lemma:residual_2}
\begin{align*}
  & \sum_{\substack{
        a \in A_{S_k}\\
        g \in C (a, S_k)}} \frac{P_{\Pi_k = a}^2}{Q_{\Pi_k =
  a}}  \frac{P_{X_1, \ldots, X_{k - 1} \setminus \Pi_k | \Pi_k}^2
  (g|a)}{Q_{X_1, \ldots, X_{k - 1} \setminus \Pi_k | \Pi_k} (g|a)}  \left(
    -  1  + \sum_{b
  \in B (a, S_k)}   
     \frac{P^2_{X_k | \Pi_k}
  (b|a)}{Q_{X_k | \Pi_k} (b|a)} \right)\\
  & \leqslant  \frac{2 c' \cdot 2^d n \log (2^d n)}{cm \varepsilon^2}
  d_{\chi^2} (P_{X_1, \ldots, X_{k - 1}}, Q_{X_1, \ldots, X_{k - 1}}, S_{k -
  1}) + \frac{4 c \cdot 2^d \log (2^d n)}{m} .
\end{align*}
\end{lemma}

\begin{proof}
    \begin{align}
      & \sum_{\substack{
        a \in A_{S_k}\\
        g \in C (a, S_k)}}
       \frac{P_{\Pi_k = a}^2}{Q_{\Pi_k = a}} \frac{P_{X_1,
      \ldots, X_{k - 1} \setminus \Pi_k | \Pi_k}^2 (g|a)}{Q_{X_1, \ldots, X_{k -
      1} \setminus \Pi_k | \Pi_k} (g|a)} \left(  
      -  1  + 
        \sum_{b \in B (a, S_k)} 
       \frac{P^2_{X_k | \Pi_k} (b|a)}{Q_{X_k | \Pi_k} (b|a)}
      \right) \\
      \leqslant &
      \sum_{\substack{
        a \in A_S\\
        g \in C (a, S_k)}} \frac{P_{\Pi_k = a}^2}{Q_{\Pi_k = a}}  \frac{P_{X_1, \ldots,
      X_{k - 1} \setminus \Pi_k | \Pi_k}^2 (g|a)}{Q_{X_1, \ldots, X_{k - 1}
      \setminus \Pi_k | \Pi_k} (g|a)}  \frac{2 c' \cdot \log (2^d n)}{mP_{\Pi_k}
      (a)} \nonumber\\
      = &
      \sum_{\substack{
        a \in A_S\\
        g \in C (a, S_k)
      }}   
         \left( - 2 P_{X_1,
      \ldots, X_{k - 1}} (a, g) + Q_{X_1, \ldots, X_{k - 1}} (a, g) \nobracket +
      \frac{P^2_{X_1, \ldots, X_{k - 1}} (a, g)}{Q_{X_1, \ldots, X_{k - 1}} (a,
      g)} \right) \cdot \frac{2 c' \cdot \log (2^d n)}{mP_{\Pi_k} (a)} \nonumber\\
      & \hspace{1cm} + \sum_{\substack{
        a \in A_S\\
        g \in C (a, S_k)
      }} (2 P_{X_1, \ldots, X_{k - 1}} (a, g) - Q_{X_1, \ldots, X_{k -
      1}} (a, g)) \cdot \frac{2 c' \cdot \log (2^d n)}{mP_{\Pi_k} (a)} \nonumber\\
      = &
      \sum_{\substack{
        a \in A_S\\
        g \in C (a, S_k)
      }} \frac{(P_{X_1, \ldots, X_{k - 1}} (a, g) - Q_{X_1, \ldots,
      X_{k - 1}} (a, g))^2}{Q_{X_1, \ldots, X_{k - 1}} (a, g)} \cdot \frac{2 c'
      \cdot \log (2^d n)}{mP_{\Pi_k} (a)} \nonumber\\
      & \hspace{1cm} + \sum_{a \in A_S} \sum_{g \in C (a, S_k)} (2 P_{X_1,
      \ldots, X_{k - 1}} (a, g) - Q_{X_1, \ldots, X_{k - 1}} (a, g)) \cdot \frac{2
      c' \cdot \log (2^d n)}{mP_{\Pi_k} (a)}  \label{eq:mark_recurrence_term}\\
      \leqslant & \frac{2 c' 2^d n \log (2^d n)}{cm \varepsilon^2} \cdot
      \sum_{\substack{
        a \in A_S\\
        g \in C (a, S_k)}}   
         \frac{(P_{X_1, \ldots,
      X_{k - 1}} (a, g) - Q_{X_1, \ldots, X_{k - 1}} (a, g))^2}{Q_{X_1, \ldots,
      X_{k - 1}} (a, g)} \nonumber\\
      & \hspace{1cm} + \sum_{a \in A_S} \sum_{g \in C (a, S_k)} P_{X_1, \ldots,
      X_{k - 1}} (a, g) \cdot \frac{4 c' \cdot \log (2^d n)}{mP_{\Pi_k} (a)}
      \nonumber\\
      = & \frac{2 c' 2^d n \log (2^d n)}{cm \varepsilon^2}  (P_{X_1, \ldots, X_{k
      - 1}}, Q_{X_1, \ldots, X_{k - 1}}, S_{k - 1}) \nonumber\\
      & \hspace{1cm} + \sum_{a \in A_S} P_{\Pi_k} (a) \cdot \frac{4 c' \cdot \log
      (2^d n)}{mP_{\Pi_k} (a)}   
         \underbrace{\sum_{g \in C
      (a, S_k)}   
         P_{X_1, \ldots, X_{k - 1}
      \setminus \Pi_k | \Pi_k} (g|a))}_{\leqslant 1} \nonumber\\
      \leqslant & \frac{2 c' 2^d n \log (2^d n)}{cm \varepsilon^2} d_{\chi^2}
      (P_{X_1, \ldots, X_{k - 1}}, Q_{X_1, \ldots, X_{k - 1}}, S_k) + \frac{4 c'
      \cdot 2^d \log (2^d n)}{m} . 
     \end{align}
\end{proof}

\subsection{A lower bound for learning a Bayes net in $\chi^2$}
\label{section:lower_bound_learning_bayes_nets}
Our lower bound relies on a family of degree-$1$ Bayes nets, with all $n - 1$
nodes sharing the same common $1$-node parent. We will set the probability of
the parent to be so imbalanced that by taking even $\exp(O(n))$ number of
samples, we still cannot obtain one sample from the rare side. In this case, it would be impossible
to observe any sample from one side of the $n - 1$ conditionals and thus,
it is information theoretically hard to obtain good estimates of these
conditional densities. Due to the multiplicative accumulation of error in these ``hidden'' conditional densities, the $\chi^2$ distance remains large despite the small parent probability. Note that, this is not a problem for $\tmop{KL}$: the error expressed in terms of $\tmop{KL}$ is linearly accumulated as compared to $\chi^2$'s multiplicative (and hence exponential) accumulation.

\begin{definition}\label{definition:hard_family}
We define a process for drawing our hard instances to analyze in  \cref{prop:minimax_risk_bayes_net}:
\begin{enumerate}
  \item {Let the prior $P \sim \pi$ be distributions such that $P = P_{X_1} \cdot  P_{X_2 |X_1} \cdots P_{X_{n} |X_1}$ and $P (X_1 = 1) = \varepsilon_0$, a parameter of our choosing.}
  
  \item {Draw $x^h$ uniformly at random from $\{0, 1\}^{n - 1}$.}
  
  \item {Based on $x^h$, set $P_{X_i |X_1 = 1} = \delta_{x^h_i}$ and as a
  consequence $P_{X_2, \ldots, X_{n} |X_1 = 1} = \prod_{i = 2}^{n}
  P_{X_i |X_1 = 1} = \delta_{x^h}$, where}
   \[ \delta_y (x) = \left\{\begin{array}{ll}
     1, & x = y\\
     0, & \text{otherwise}
   \end{array}\right. . \] 
  \item {Let $P_{X_2, \ldots, X_{n} |X_1 = 0} = U_{n-1}$ be uniform on $\{0,1\}^{n-1}$ in the other case.}
\end{enumerate}

\end{definition}

Intuitively, we are merely hiding a particular point $x^h$ from the learners; each distribution in $\pi$, when conditioned on $X_1 = 1$ will concentrate their mass on one point in the simplex (deterministic). While it takes only one sample (with $x_1=1$) to learn, no learner can get one with less than $O( 1/\varepsilon_0 )$ number of samples.
We state the main result below.
\begin{proposition}\label{prop:minimax_risk_bayes_net}
  Minimax risk of estimating a degree-$1$ Bayes net $P$ in $\chi^2$ is at least $\Omega(\varepsilon)$ when the number of samples $m \leqslant O(2^{n/2}/\varepsilon)$. In particular, the family of distributions in \cref{definition:hard_family} takes at least $\Omega (2^{n/2} / \varepsilon)$ to learn to $\varepsilon$.
\end{proposition}

By a standard Lagrange multiplier calculation, as inspired by the proof of \cite[Lemma 5]{DBLP:conf/colt/KamathOPS15}, we have these useful facts. We will use them to prove \cref{prop:minimax_risk_bayes_net}.

\begin{fact}\label{fact:lagrange_optima} Let
  $ q_i,a_i \geqslant 0, i \in [k]\text{, such that } \sum_{i = 1}^k q_i \leqslant 1
  $. Then the quantity $\sum_{i = 1}^k \frac{a_i}{q_i}$ is minimized when
  $q_i \propto \sqrt{a_i}$.
\end{fact}

\begin{fact}\label{fact:KKT_optima_hard_instance}
  Optima of the following form is obtained when $Q^2_i=\frac{1}{2^{n-1}}, i=1, \cdots, n-1$:
  \[\min_{Q \in \mathbb{R}^{n-1}} \sum_{i = 1}^{2^{n - 1}} \frac{1}{Q_i} \cdot \frac{1}{2^{n - 1}},
    \text{s.t.} \sum_{i=1}^{2^{n-1}} Q^2_i = 1.\]
\end{fact}

\begin{proof}
  To see this, we simply verify the K.K.T. condition \cite{opttext} of the constrained optimization problem:
  \[ \sum_{i = 1}^{2^{n - 1}} \frac{1}{Q_i} \cdot \frac{1}{2^{n - 1}} + \lambda
    \left( \sum_{i = 1}^{2^{n - 1}} Q_i^2 - 1 \right). \]
  Necessary conditions:
  \begin{align*}
    - \frac{1}{Q_i^2} \cdot \frac{1}{2^{n - 1}} + 2 \lambda Q_i = 0, \forall i; \hspace{0.5cm}
    \sum_{i = 1}^{2^{n - 1}} Q_i^2 - 1 = 0.
    \end{align*}
  Solving the first equation gives
  $ Q_i = \frac{1}{2^{n / 3} \lambda^{1 / 3}}; $
  and since all $Q_i$s are equal, $Q^{\ast}$ ought to be uniform. Since
  $\frac{1}{Q_i}$ and $Q_i^2$ are both convex, we have that the necessary conditions
  are also sufficient for global minima.      
\end{proof}

  \begin{proof}
  [Proof of \cref{prop:minimax_risk_bayes_net}]
  \label{proof:prop:minimax_risk_bayes_net}
    Let $s_i = (X_{1, i}, \ldots, X_{n, i})$ be the $i^{th}$ sample, denote event $S = \{X_{1, i} = 0, i\in[m]\}$, the probability is thus $\Pr [S] = (1 - \varepsilon_0)^m
    \geqslant e^{- \frac{1}{2} \varepsilon_0 m} \geqslant 1 - \frac{1}{2}
    \varepsilon_0 m$. We will condition on $S$ being true during the computation. Let $\mathcal{A}_m$ denote the set of deterministic algorithms taking $m$ samples from $P$. Then, for $R_{\chi^2}(m)$ being the minimax risk over $\mathcal{A}_m$,
    \begin{eqnarray}
      R_{\chi^2} (m) & = & \inf_{Q \in \mathcal{A}_m} \sup_{P \in \mathcal{P}}
      \mathbb{E}_{s_1, \ldots, s_m \sim P} \left[ d_{\chi^2} \left( P,
      Q_{\orgvec{s}} \right) \right] \nonumber\\
      & \geqslant & \inf_{Q \in \mathcal{A}_m} \mathbb{E}_{P \sim \pi}
      \mathbb{E}_{\orgvec{s} \sim P^{\otimes m}} \left[ d_{\chi^2} \left( P,
      Q_{\orgvec{s}} \right) \right] \nonumber\\
      & \geqslant & \inf_{Q \in \mathcal{A}_m} \mathbb{E}_{P \sim \pi}
      \mathbb{E}_{\orgvec{s} \sim P^{\otimes m}} \left[ d_{\chi^2} \left( P,
      Q_{\orgvec{s}} \right) |S \right] \Pr [S] \nonumber\\
      & = & \inf_{Q \in \mathcal{A}_m} \underset{P \sim \pi}{\mathbb{E}}
      \underset{\orgvec{s} \sim P^{\otimes m}}{\mathbb{E}} \left[ - 1 + \sum_{x_1,
      \ldots, x_n} \frac{P^2 (x_1, \ldots, x_n)}{Q_{\orgvec{s}} (x_1, \ldots,
      x_n)} \mid S \right] \Pr [S] .  \label{eq:minimax_risk_to_be_connected}
    \end{eqnarray}

    Denote for convenience, $P_0 = P_{X_1, \ldots, X_{n - 1} |X_n = 0}$; $P_1 =
    P_{X_1, \ldots, X_{n - 1} |X_n = 1}$; $Q_{0, \orgvec{s}} = Q_{\orgvec{s}}
    (X_1, \ldots, X_{n - 1} |X_n = 0)$; $Q_{1, \orgvec{s}} = Q_{\orgvec{s}} (X_1,
    \ldots, X_{n - 1} |X_n = 1)$. We focus on the inner summation and lower bound
    them separately; and before that, we need a separate auxiliary tool -- using \autoref{fact:lagrange_optima} above, we can show that, for any fixed $P$ and $Q_{\orgvec{s}}$,
    \begin{align}
      & \frac{P^2  (x_1 = 0)}{Q_{\orgvec{s}}  (x_1 = 0)} + \frac{P^2  (x_1 =
      1)}{Q_{\orgvec{s}}  (x_1 = 1)}  \left( 1 + d_{\chi^2} \left( P_1, Q_{1,
      \orgvec{s}} \right) \right) \nonumber\\
      \geqslant & \frac{P^2  (x_1 = 0)}{\frac{P (x_1 = 0)}{P (x_1 = 0) + P (x_1 =
      1) \sqrt{1 + d_{\chi^2} (P_1, Q_1)}}} + \frac{P^2  (x_1 = 1)  (1 + d_{\chi^2} (P_1, Q_1))}{\frac{P
      (x_1 = 1) \sqrt{1 + d_{\chi^2} (P_1, Q_1)}}{P (x_1 = 0) + P (x_1 = 1)
      \sqrt{1 + d_{\chi^2} (P_1, Q_1)}}} \nonumber\\
      \geqslant & \left( P (x_1 = 0) + P (x_1 = 1) \sqrt{1 + d_{\chi^2} (P_1,
      Q_1)} \right)^2  \label{eq:lagrange_optima}
    \end{align}

    \begin{align}
      \Longrightarrow & - 1 + \sum_{x_1, \ldots, x_n} \frac{P^2 (x_1, \ldots, x_{n
      - 1}, x_n)}{Q_{\orgvec{s}} (x_1, \ldots, x_{n - 1}, x_n)} \nonumber\\
      = & - 1 + \sum_{x_1, \ldots, x_{n - 1}} \frac{P^2  (x_n = 0) P^2 (x_1,
      \ldots, x_{n - 1} |x_n = 0)}{Q_{\orgvec{s}}  (x_n = 0) Q_{\orgvec{s}} (x_1,
      \ldots, x_{n - 1} |x_n = 0)} \nonumber\\
      & \hspace{0.6cm} + \sum_{x_1, \ldots, x_{n - 1}} \frac{P^2  (x_n = 1) P^2
      (x_1, \ldots, x_{n - 1} |x_n = 1)}{Q_{\orgvec{s}}  (x_n = 1) Q_{\orgvec{s}}
      (x_1, \ldots, x_{n - 1} |x_n = 1)} \nonumber\\
      = & - 1 + \frac{P^2  (x_n = 0)}{Q_{\orgvec{s}}  (x_n = 0)}  \left( 1 +
      d_{\chi^2} \left( P_0, Q_{0, \orgvec{s}} \right) \right) + \frac{P^2  (x_n =
      1)}{Q_{\orgvec{s}}  (x_n = 1)}  \left( 1 + d_{\chi^2} \left( P_1, Q_{1,
      \orgvec{s}} \right) \right) \nonumber\\
      \geqslant & - 1 + \frac{P^2  (x_n = 0)}{Q_{\orgvec{s}}  (x_n = 0)} +
      \frac{P^2  (x_n = 1)}{Q_{\orgvec{s}}  (x_n = 1)} \left( 1 + d_{\chi^2}
      \left( P_1, Q_{1, \orgvec{s}} \right) \right) \nonumber\\
      \geqslant & - 1  +  \left( 
       P (x_n = 0) + P (x_n = 1) \sqrt{1 + d_{\chi^2} \left( P_1,
      Q_{1, \orgvec{s}} \right)} \right)^2  \label{eq:lower_bound_variational}\\
      = & - 1 + \left( (1 - \varepsilon_0) + \varepsilon_0  \sqrt{1 + d_{\chi^2}
      \left( P_1, Q_{1, \orgvec{s}} \right)} \right)^2 \nonumber\\
      = & - 1 + \left( 1 + \varepsilon_0  \left( \sqrt{1 + d_{\chi^2} \left( P_1,
      Q_{1, \orgvec{s}} \right)} - 1 \right) \right)^2 \nonumber\\
      = & \hspace{0.27em} 2 \varepsilon_0  \left( \sqrt{1 + d_{\chi^2} \left( P_1,
      Q_{1, \orgvec{s}} \right)} - 1 \right) + \left( \varepsilon_0  \left(
      \sqrt{1 + d_{\chi^2} \left( P_1, Q_{1, \orgvec{s}} \right)} - 1 \right)
      \right)^2 \nonumber\\
      \geqslant & \hspace{0.27em} 2 \varepsilon_0  \left( \sqrt{1 + d_{\chi^2}
      \left( P_1, Q_{1, \orgvec{s}} \right)} - 1 \right) . 
      \label{eq:variation_lb_to_connect}
    \end{align}

    For any fixed $Q_{\orgvec{s}} (x_1, \ldots, x_n)$, we can compute
    $Q_{\orgvec{s}}  (x_n = 0), Q_{\orgvec{s}}  (x_n = 1)\text{, and } d_{\chi^2} \left( P_1,
    Q_{1, \orgvec{s}} \right)$; in other words, they are also fixed, given $\orgvec{s}$. Then
    a lower bound can be obtained via a variation argument in
    \eqref{eq:lower_bound_variational} via \eqref{eq:lagrange_optima}. Connecting 
    \eqref{eq:minimax_risk_to_be_connected}, and \eqref{eq:variation_lb_to_connect},
    we continue with the following expression,
    \begin{eqnarray}
      \frac{R_{\chi^2} (m)}{\Pr [S]} & \geqslant & \inf_{Q \in \mathcal{Q}}
      \underset{\begin{array}{c}
        P \sim \pi\\
        \orgvec{s} \sim P^{\otimes m}
      \end{array}}{\mathbb{E}} \left[ 2 \varepsilon_0 (\sqrt{1 + d_{\chi^2} \left(
      P_1, Q_{1, \orgvec{s}} \right)} - 1) |S \right] \nonumber\\
      & = & \inf_{Q \in \mathcal{Q}} \underset{P_1 \sim \pi}{\mathbb{E}}
      \underset{\orgvec{s} \sim P^{\otimes m} |S}{\mathbb{E}} \left[ 2
      \varepsilon_0  \left( \sqrt{1 + d_{\chi^2} \left( P_1, Q_{1, \orgvec{s}}
      \right)} - 1 \right) \right] \nonumber\\
      & = & \inf_{Q \in \mathcal{Q}} \underset{x^h \sim U_{n - 1}}{\mathbb{E}}
      \underset{\orgvec{s} \sim \tilde{U}_{n - 1}^{\otimes m}}{\mathbb{E}} \left[
      2 \varepsilon_0  \left( \sqrt{1 + d_{\chi^2} \left( P_1, Q_{1, \orgvec{s}}
      \right)} - 1 \right) \right]  \label{eq:change_notation_clarification}\\
      & = & \inf_{Q \in \mathcal{Q}} \underset{\orgvec{s} \sim \tilde{U}_{n -
      1}^{\otimes m}}{\mathbb{E}} \underset{x^h \sim U_{n - 1}}{\mathbb{E}} \left[
      2 \varepsilon_0  \left( \sqrt{1 + d_{\chi^2} \left( P_1, Q_{1, \orgvec{s}}
      \right)} - 1 \right) \right]  \label{eq:swap_independence}\\
      & \geqslant & \left( \underset{\orgvec{s} \sim \tilde{U}_{n - 1}^{\otimes
      m}}{\mathbb{E}} \inf_{Q \in \mathcal{Q}} \underset{x^h \sim U_{n -
      1}}{\mathbb{E}} \left[ 2 \varepsilon_0  \left( \sqrt{\sum_x \frac{P_1^2
      (x)}{Q_{1, \orgvec{s}} (x)}} - 1 \right) \right] \right) 
      \label{eq:after_swap}\\
      & \geqslant & \left( \inf_{Q^{\ast} \in \Delta^{2^{n - 1}}} \underset{x^h
      \sim U_{n - 1}}{\mathbb{E}} \left[ 2 \varepsilon_0  \left( \sqrt{\sum_x
      \frac{P_1^2 (x)}{Q^{\ast} (x)}} - 1 \right) \right] \right) 
      \label{eq:fixed_Q_lb}\\
      & = & \inf_{Q^{\ast} \in \Delta^{2^{n - 1}}} \underset{\tmmathbf{\alpha}
      \triangleq (\alpha_1, \ldots, \alpha_{n - 1}) \sim U_{n - 1}}{\mathbb{E}}
      \left[ 2 \varepsilon_0  \left( \sqrt{\frac{1}{Q^{\ast} (\tmmathbf{\alpha})}}
      - 1 \right) \right] \nonumber\\
      & = & 2 \varepsilon_0  \left( 2^{\frac{n - 1}{2}} - 1 \right) . \label{eq:final_optima}
    \end{eqnarray}

    Since we are only changing $P_1$ (or $x^h$) in the construction, we can replace
    $P \sim \pi$ with $P_1 \sim \pi$ (or $x^h \sim U_{n - 1}$); note that there is no
    sample with $X_1 = 1$ in $\orgvec{s}$, and thus $\orgvec{s}$ is merely samples
    drawn from uniform distribution with all their corresponding $X_1 = 0$ and this is
    what we mean by $\orgvec{s} \sim \tilde{U}_{n-1}^{\otimes m}$ in \eqref{eq:change_notation_clarification}.
    Therefore, $P_1$ or $x^h$ is independent with $\orgvec{s}$, and we can swap
    the expectation in \eqref{eq:swap_independence}; and in \eqref{eq:after_swap}, we lower bound the 
    expectation as the learner can first observe the samples $\orgvec{s}$ before choosing 
    the algorithm from $\mathcal{Q}$. But in any case, it is fixed before the last expectation, and hence \eqref{eq:fixed_Q_lb} follows.
    As we assume the learning algorithm $Q$ is deterministic, for a fixed $\orgvec{s}$, $Q_{1, \orgvec{s}}$ is
    also fixed. We obtain \eqref{eq:final_optima} through \autoref{fact:KKT_optima_hard_instance}.

    In the end, we have that
    \[ R_{\chi^2} (m) \geqslant 2 \varepsilon_0  \left( 2^{\frac{n - 1}{2}} - 1
      \right) \Pr [S] \geqslant \varepsilon_0 2^{\frac{n}{2}}  (1 - \varepsilon_0
      m) . \]
    By setting $\varepsilon_0 = \frac{2 \varepsilon}{2^{n/2}}$, we can see
    that if $m \leqslant \frac{1}{4\eps} 2^{\frac{n}{2}} $, then
    $R_{\chi^2} (m) \geqslant 2 \varepsilon - \frac{4 \varepsilon^2}{2^{n / 2}} m
    \geqslant \varepsilon$.
  \end{proof}

\section{Testing maximum in-degree of Bayes nets}\label{sec:degree_tester}

\begin{algorithm2e}[htp]
  \SetKwComment{Comment}{/* }{ */}\small
  \SetKwInOut{Input}{Input}
  \DontPrintSemicolon
  \Input{Sample access to distribution $P$, accuracy parameter $\varepsilon$ and a degree-$d$ DAG $G$.}
  Learn $P$  with $\varepsilon$, $G$ and $\frac{2^{d} n^2 \log (2^d n)}{\varepsilon^2}$ samples via \cref{algo:near_proper_learning_chi_square}: obtaining an estimate $\tilde{Q}$, and an $O(\varepsilon^2)$-effective support set $\mathcal{A}$ via \cref{algo:majority_set_learning}.\;
  Draw a multiset $S$ of $\tmop{Poisson}(m)$ samples from $P$, where $m=\frac{2^{n/2}}{\varepsilon^2}$.\;
  Call \cite[Algorithm 1]{DBLP:conf/soda/DaskalakisKW18} and return  $P_{\mathcal{A}}$, $\tilde{Q}_{\mathcal{A}}$, $S$, $\varepsilon$.\;
  \caption{Testing $P$ is a degree-$d$ DAG $G$\label{alg:test-degree-d-DAG}}
\end{algorithm2e}

\begin{thm}
  \label{theorem:graph-test-general-distribution} Given an unknown distribution
  $P$, and a maximum degree-$d$ graph $G$ supported on $\{0, 1\}^n$, it takes
  at most $O \left( \max \left( \frac{2^{n / 2}}{\varepsilon^2}, \frac{2^d n^2 d
  \log (n)}{\varepsilon^2} \right) \right)$ i.i.d. samples to test whether
  $d_H  (P, G) = 0 \infixor d_H  (P, G) \geqslant \varepsilon$, with 
  probability $\geqslant 2 / 3$.
  
  Furthermore, testing whether $P$ is Markov w.r.t. any max degree-$d$
  graphs with success probability at least $2 / 3$, takes at most $O \left(
  \max \left( \frac{2^{n / 2}}{\varepsilon^2}, \frac{2^d n^2 d \log
  (n)}{\varepsilon^2} \right) \cdot \log (n^{dn}) \right)$ samples.
\end{thm}

\begin{proof}
  We prove by analyzing \cref{alg:test-degree-d-DAG}. By the guarantee of the underlying tester in \cite[Algorithm 1]{DBLP:conf/soda/DaskalakisKW18}, it suffices to verify the following:
  \begin{itemizedot}
    \item {\tmem{Soundness:}} If $d_H  (P, G) \geqslant \varepsilon$ (it is far from any Bayes nets of graph $G$), then $d_H  (P_{\mathcal{A}}, \Tilde{Q}_{\mathcal{A}}) \geqslant \Omega(\varepsilon$);
    
    \item {\tmem{Correctness:}} If $d_H  (P, G) = 0$, then $d_{\chi^2}  (P_{\mathcal{A}}, \tilde{Q}_{\mathcal{A}}) \leqslant O (\varepsilon^2)$; and we have this from \autoref{theorem:near_proper_chi_squared_learning_bayes}.
  \end{itemizedot}
  Roughly speaking, we can pretend $P$ and $\tilde{Q}$ are supported only on $\mathcal{A}$, and since $|\mathcal{A}| \leqslant 2^{n}$, $O(\sqrt{2^{n}})/\varepsilon^2$ samples suffice for testing.
  For soundness, by \autoref{theorem:near_proper_chi_squared_learning_bayes}, we have
  \[\tilde{Q} (\tilde{S}) \geqslant 1 - O (\varepsilon^2),
  \infixand P (\tilde{S}) \geqslant 1 - O (\varepsilon^2).\] Since
  \[d_H^2 (P, Q) = d_H^2 (P_{\mathcal{A}}, Q_{\mathcal{A}}) + d_H^2
        (P_{\bar{\mathcal{A}}}, Q_{\bar{\mathcal{A}}}), \infixand d_H^2
        (P_{\bar{\mathcal{A}}}, Q_{\bar{\mathcal{A}}}) \leqslant d_{\tmop{TV}}
        (P_{\bar{\mathcal{A}}}, Q_{\bar{\mathcal{A}}}) \leqslant \frac{1}{2}  (P
        (\bar{\mathcal{A}}) + Q (\bar{\mathcal{A}})) \leqslant O (\varepsilon^2),\] we
        have that $d_H^2 (P_{\mathcal{A}}, Q_{\mathcal{A}}) \geqslant \Omega
        (\varepsilon^2)$.
    
  Since it costs an extra $O (\log (1 / \delta))$ to amplify the success
  probability to $1 - \delta$ for each test, we will run amplified accurate
  tests on all $n^{O (dn)}$ possible maximum in-degree-$d$ graphs and follow
  up with a union bound of $1 - \delta \cdot n^{O (dn)}$. In particular, we
  set $\delta = \frac{1}{n^{dn}}$, which brings an additional $O (\log
  (n^{dn}))$ factor to the overall sample complexity, and thus, it gives us a
  tester for maximum in-degree-$d$ graphs with sample complexity $O \left(
  \max \left( \frac{2^{n / 2}}{\varepsilon^2}, \frac{2^d n^2d \log
  (n)}{\varepsilon^2} \right) \cdot \log (n^{dn}) \right)$.
\end{proof}

\subsection{Extending \autoref{theorem:graph-test-general-distribution} to TV distance}\label{section:TV_distance_extension}
  While our result in \autoref{theorem:graph-test-general-distribution}
  already implies a tester in $d_{\tmop{TV}}$, with our near-proper learner in
  $d_{\chi^2}$ for bounded degree Bayes net, it also implies a similar
  graphical tester in $d_{\tmop{TV}}$ analogous to
  \autoref{theorem:graph-test-general-distribution}, where the shifting of masses
  is unnecessary (see {\cite[Remark 1]{DBLP:conf/nips/AcharyaDK15}}),
  i.e., the additional requirement of $Q (\tilde{S}) \geqslant 1 - O
  (\varepsilon^2)$ is no longer necessary in the case of $\tmop{TV}$; and we can also weaken requirement on $P(\tilde{S})$: $P(\Tilde{S}) \geqslant 1 - O(\varepsilon)$.

  To see this, we only need to verify that $d_{\tmop{TV}} (P_\mathcal{A}, Q_{\mathcal{A}}) \geqslant \Omega (\varepsilon)$ in the case of soundness.
  Assuming that $d_{\tmop{TV}} (P, Q) > 10 \varepsilon$ and $P (S) > 1 -
  \varepsilon$, we analyze the two cases,
  \begin{itemizeminus}
    \item When $Q (S) < 1 - 2 \varepsilon$, we have $P (S^c) \leqslant
    \varepsilon, Q (S^c) > 2 \varepsilon$, and thus
    \[\frac{1}{2}  \sum_{i \in S} |P_i - Q_i |
       \geqslant \frac{1}{2} \left| \sum_{i \in S} (P_i - Q_i) \right| = \frac{1}{2}  (P (S) - Q (S))
       > \frac{1}{2}  (1 - \varepsilon - (1 - 2
       \varepsilon)) = \frac{\varepsilon}{2} .\]
    \item When $Q (S) < 1 - 2 \varepsilon$, similarly,
    \[\frac{1}{2}  \sum_{i \in S} |P_i - Q_i | =
        \frac{1}{2}  \sum_{i \in \Omega} |P_i - Q_i | - \frac{1}{2} \sum_{i\nin S} |P_i - Q_i | > \frac{1}{2}  (10 \varepsilon - (\varepsilon + 2 \varepsilon)) = \frac{7}{2} \varepsilon .\]
    \end{itemizeminus}
  In both cases, we have $d_{\tmop{TV}} (P_S, Q_S) %
  \geqslant \Omega
  (\varepsilon)$. Nevertheless, we note that the same technique does not work for Hellinger, and thus requires a slightly stronger guarantee.

\section{Conclusion and future directions}
In this paper, we provided (nearly) tight sample complexity bounds for testing the maximum in-degree of an unknown Bayesian network. Along the way, we obtained several results of independent interest, including a near-proper learner for Bayesian networks under $\chi^2$ divergence, and a high-probability $\chi^2$ learning algorithm (for arbitrary discrete distributions).

Our results raise two interesting future directions. The first is to generalize our testing result to the more general question of maximum degree-$d$ testing \emph{under maximum degree-$k$ assumption}, where $k > d$ are both input parameters; in particular, our results correspond to $k=n$. The second is to either strengthen our high-probability $\chi^2$ learning bound to obtain an \emph{additive} $\log(1/\delta)$ dependence (as is known for total variation distance learning), instead of a multiplicative one; or to show that such a multiplicative dependence on $\log(1/\delta)$ is necessary. We note that such a result is not known even for the weaker KL divergence learning.

\section*{Acknowledgment}
Yang would like to thank Philips George John for the helpful discussions, and for suggesting the $\delta$ function notation in the lower bound analysis.

\bibliographystyle{alpha}
\bibliography{biblio}

\newcommand{\etalchar}[1]{$^{#1}$}
\begin{thebibliography}{CDDK22}

\bibitem[ABDK18]{AcharyaBDK18}
Jayadev Acharya, Arnab Bhattacharyya, Constantinos Daskalakis, and Saravanan
  Kandasamy.
\newblock Learning and testing causal models with interventions.
\newblock In {\em NeurIPS}, pages 9469--9481, 2018.

\bibitem[ADK15]{DBLP:conf/nips/AcharyaDK15}
Jayadev Acharya, Constantinos Daskalakis, and Gautam Kamath.
\newblock Optimal testing for properties of distributions.
\newblock In {\em {NIPS}}, pages 3591--3599, 2015.

\bibitem[BBC{\etalchar{+}}20]{BezakovaBCSV20}
Ivona Bez{\'{a}}kov{\'{a}}, Antonio Blanca, Zongchen Chen, Daniel Stefankovic,
  and Eric Vigoda.
\newblock Lower bounds for testing graphical models: Colorings and
  antiferromagnetic ising models.
\newblock {\em J. Mach. Learn. Res.}, 21:25:1--25:62, 2020.

\bibitem[BCY22]{DBLP:journals/corr/abs-2204-08690}
Arnab Bhattacharyya, Cl{\'{e}}ment~L. Canonne, and Joy~Qiping Yang.
\newblock Independence testing for bounded degree bayesian network.
\newblock {\em CoRR}, abs/2204.08690, 2022.

\bibitem[BGPV21]{DBLP:conf/stoc/0001GPV21}
Arnab Bhattacharyya, Sutanu Gayen, Eric Price, and N.~V. Vinodchandran.
\newblock Near-optimal learning of tree-structured distributions by
  {C}how--{L}iu.
\newblock In {\em {STOC}}, pages 147--160. {ACM}, 2021.

\bibitem[BV14]{opttext}
Stephen~P. Boyd and Lieven Vandenberghe.
\newblock {\em Convex Optimization}.
\newblock Cambridge University Press, 2014.

\bibitem[Can20]{Canonne:NoteLearningDistributions}
Clément~L. Canonne.
\newblock A short note on learning discrete distributions, 2020.

\bibitem[CDDK22]{ChooDDK22}
Davin Choo, Yuval Dagan, Constantinos Daskalakis, and Anthimos~Vardis Kandiros.
\newblock Learning and testing latent-tree ising models efficiently.
\newblock {\em CoRR}, abs/2211.13291, 2022.

\bibitem[CDKL22]{CDKL22}
Cl{\'{e}}ment~L. Canonne, Ilias Diakonikolas, Daniel~M. Kane, and Sihan Liu.
\newblock Near-optimal bounds for testing histogram distributions.
\newblock {\em CoRR}, abs/2207.06596, 2022.

\bibitem[CDKS17]{canonne2017testing}
Cl{\'e}ment~L Canonne, Ilias Diakonikolas, Daniel~M Kane, and Alistair Stewart.
\newblock Testing bayesian networks.
\newblock In {\em Conference on Learning Theory}, pages 370--448. PMLR, 2017.

\bibitem[DDK17]{DaskalakisDK17}
Constantinos Daskalakis, Nishanth Dikkala, and Gautam Kamath.
\newblock Concentration of multilinear functions of the ising model with
  applications to network data.
\newblock In {\em {NIPS}}, pages 12--23, 2017.

\bibitem[DDK19]{DaskalakisD019}
Constantinos Daskalakis, Nishanth Dikkala, and Gautam Kamath.
\newblock Testing ising models.
\newblock {\em {IEEE} Trans. Inf. Theory}, 65(11):6829--6852, 2019.

\bibitem[Dia16]{Diakonikolas16}
Ilias Diakonikolas.
\newblock Learning structured distributions.
\newblock In {\em Handbook of big data}, Chapman \& Hall/CRC Handb. Mod. Stat.
  Methods, pages 267--283. CRC Press, Boca Raton, FL, 2016.

\bibitem[DKW18]{DBLP:conf/soda/DaskalakisKW18}
Constantinos Daskalakis, Gautam Kamath, and John Wright.
\newblock Which distribution distances are sublinearly testable?
\newblock In {\em {SODA}}, pages 2747--2764. {SIAM}, 2018.

\bibitem[DL01]{DevroyeL01}
Luc Devroye and G\'{a}bor Lugosi.
\newblock {\em Combinatorial methods in density estimation}.
\newblock Springer Series in Statistics. Springer-Verlag, New York, 2001.

\bibitem[DP17]{DaskalakisP17}
Constantinos Daskalakis and Qinxuan Pan.
\newblock Square hellinger subadditivity for bayesian networks and its
  applications to identity testing.
\newblock In {\em {COLT}}, volume~65 of {\em Proceedings of Machine Learning
  Research}, pages 697--703. {PMLR}, 2017.

\bibitem[FLNP00]{friedman2000using}
Nir Friedman, Michal Linial, Iftach Nachman, and Dana Pe'er.
\newblock Using bayesian networks to analyze expression data.
\newblock In {\em Proceedings of the fourth annual international conference on
  Computational molecular biology}, pages 127--135, 2000.

\bibitem[GNS18]{GangradeNS18}
Aditya Gangrade, Bobak Nazer, and Venkatesh Saligrama.
\newblock Two-sample testing can be as hard as structure learning in ising
  models: Minimax lower bounds.
\newblock In {\em {ICASSP}}, pages 6931--6935. {IEEE}, 2018.

\bibitem[Hec98]{heckerman1998tutorial}
David Heckerman.
\newblock {\em A tutorial on learning with Bayesian networks}.
\newblock Springer, 1998.

\bibitem[KOPS15]{DBLP:conf/colt/KamathOPS15}
Sudeep Kamath, Alon Orlitsky, Dheeraj Pichapati, and Ananda~Theertha Suresh.
\newblock On learning distributions from their samples.
\newblock In {\em {COLT}}, volume~40 of {\em {JMLR} Workshop and Conference
  Proceedings}, pages 1066--1100. JMLR.org, 2015.

\bibitem[NL19]{NeykovL19}
Matey Neykov and Han Liu.
\newblock Property testing in high-dimensional {I}sing models.
\newblock {\em Ann. Statist.}, 47(5):2472--2503, 2019.

\bibitem[Pea88]{pearl1988probabilistic}
Judea Pearl.
\newblock {\em Probabilistic reasoning in intelligent systems: networks of
  plausible inference}.
\newblock Morgan kaufmann, 1988.

\bibitem[Pea95]{pearl1995bayesian}
Judea Pearl.
\newblock From bayesian networks to causal networks.
\newblock {\em Mathematical models for handling partial knowledge in artificial
  intelligence}, pages 157--182, 1995.

\bibitem[SPP{\etalchar{+}}05]{sachs2005causal}
Karen Sachs, Omar Perez, Dana Pe'er, Douglas~A Lauffenburger, and Garry~P
  Nolan.
\newblock Causal protein-signaling networks derived from multiparameter
  single-cell data.
\newblock {\em Science}, 308(5721):523--529, 2005.

\bibitem[WJ08]{wainwright2008graphical}
Martin~J Wainwright and Michael~I Jordan.
\newblock Graphical models, exponential families, and variational inference.
\newblock {\em Foundations and Trends{\textregistered} in Machine Learning},
  1(1--2):1--305, 2008.

\end{thebibliography}

\end{document}